\documentclass{article}

\usepackage[preprint]{neurips_2026}

\usepackage[utf8]{inputenc}
\usepackage[T1]{fontenc}
\usepackage[hypertexnames=false,bookmarks=false,breaklinks=true,hidelinks]{hyperref}
\usepackage{url}
\usepackage{booktabs}
\usepackage{amsfonts}
\usepackage{amsmath,amssymb,amsthm}
\usepackage{nicefrac}
\usepackage{microtype}
\usepackage{xcolor}
\usepackage{graphicx}
\usepackage{tabularx}
\usepackage{ltablex}
\usepackage{multirow}
\usepackage{array}
\usepackage{float}
\newtheorem{lemma}{Lemma}
\newtheorem{theorem}{Theorem}
\newtheorem{corollary}{Corollary}

\newcommand{\Spart}{S_{\mathrm{part}}}
\newcommand{\Setoe}{S_{\mathrm{e2e}}}
\newcommand{\Apart}{A_{\mathrm{part}}}
\newcommand{\Aetoe}{A_{\mathrm{e2e}}}
\newcommand{\famhat}{\hat{k}_R}
\newcommand{\dfgap}{deployment-fidelity gap}

\hypersetup{
  pdftitle={Partition Scores Are Not System Scores: Deployment-Fidelity Gaps in Decomposed Algorithm Selection},
  pdfauthor={Jiachen Zhang, Yu Tang, Li Zhu},
  pdfsubject={Evaluation methodology for decomposed algorithm-selection systems},
  pdfkeywords={algorithm selection, AutoML, deployment evaluation, oracle metrics}
}

\title{Partition Scores Are Not System Scores:\\
  Deployment-Fidelity Gaps in Decomposed Algorithm Selection}

\author{%
  Jiachen Zhang\thanks{Corresponding author: \texttt{jiaz@uoregon.edu}} \\
  Department of Computer Science \\
  University of Oregon
  \And
  Yu Tang \\
  Huazhong University of Science \\
  and Technology
  \And
  Li Zhu \\
  City University of Hong Kong
}

\begin{document}

\maketitle

\begin{abstract}
Oracle-style quantities, including virtual best solvers,
selected-portfolio VBS, virtual-best encodings, and best-in-family
summaries, are widely reported as upper bounds on what a deployable
selector could achieve. In decomposed algorithm selection, an
analogous partition-level score grants an oracle choice of the best
algorithm within the selected family; once the family selector is
fixed, the deployable system must replace that within-family oracle
with a learned within-family selector.

We define the \emph{\dfgap{}} $G(R)$ as the difference between
partition-level and deployable end-to-end utility and derive two
accounting consequences: a per-instance margin--regret stability
condition that tells us when a partition-time family choice is
deployment-optimal, and a sharp partition-only identification interval
that, when it strictly crosses zero, prevents the partition-level
report from certifying the deployable winner.

Across five public algorithm-selection benchmarks spanning tabular
AutoML and combinatorial CSP/SAT, every decomposed pipeline has
positive $G(R)$, ranging from $0.012$ on TabZilla to $0.13$ on
PROTEUS-2014. Four of ten decomposed-versus-flat decisions have
sign-changing point estimates; on PROTEUS-2014, a $33$-point partition
advantage shrinks to a $20$-point end-to-end advantage. A
training-side validation gap-correction diagnostic recovers the
point-estimate deployable sign on all four sign-changing cells; it is
a reporting aid, not a substitute for direct end-to-end evaluation.
Partition and end-to-end scores should be reported side by side.
\end{abstract}

\section{Introduction}
\label{sec:intro}

Algorithm selection (AS) systems increasingly use decomposed decisions:
rather than selecting directly from a large portfolio, they first select
an algorithm family and then a concrete algorithm within it
\citep{kerschke2019automated,bischl2016aslib}. Decomposition reduces a
hard multiclass problem and exposes interpretable structure, but its
evaluation can differ from its deployment.

A common partition-level evaluation asks whether the selected family
contains a strong algorithm and then grants the best member of that
family. This upper-bounds the partition, but it is not an end-to-end
system: deployment must use a learned within-family selector. We call
the discrepancy the \emph{\dfgap{}}:
\begin{equation}
  G(R) \;=\; \Spart(R) \;-\; \Setoe(R),
  \label{eq:gap-intro}
\end{equation}
where $\Spart$ uses the within-family oracle and $\Setoe$ uses the
deployable selector. $G(R)$ is the utility lost when the oracle is
replaced by the available within-family selector.

Across five public AS benchmarks (B1: TabZilla, B2: TabRepo,
B3: TALENT, B4: MAXSAT-PMS-2016, B5: PROTEUS-2014) spanning tabular
AutoML \citep{mcelfresh2023tabzilla,salinas2024tabrepo,ye2024talent}
and combinatorial search \citep{bischl2016aslib,hurley2014proteus},
all decomposed pipelines show positive \dfgap{}s, from $1.2$ points on
TabZilla to $13$ points on PROTEUS-2014. The gap changes conclusions:
on TabRepo and MAXSAT-PMS-2016, partition-level evaluation favours
decomposition over flat selection but end-to-end point estimates change
sign, with one B2 cell and both B4 cells having small end-to-end
margins; on
TabZilla the advantage is almost absorbed but remains positive
end-to-end; on PROTEUS-2014 a $33$-point partition advantage
shrinks to $20$ points end-to-end ($\rho \approx 0.41$). Overall, the
safety condition fails in $4/10$ decomp-vs-flat deployment decisions,
whereas all $5$ pairwise-vs-multiclass outer-selector comparisons are
low-signal. The dominant risk is therefore whether the within-family
deployment step is included, not which outer selector is used.

Prior work uses oracle baselines such as virtual best solvers (VBS) as
upper bounds \citep{xu2008satzilla,bischl2016aslib,cameron2016vbs},
and recent work studies split- and scale-induced benchmarking pitfalls
\citep{petelin2025pitfalls}. We study a different, \emph{nested}
oracle: the within-family oracle introduced after a family selector has
already acted. The novelty is not that oracle scores are optimistic,
but that a common decomposition score can replace one deployable stage
with an oracle and change comparisons through $G(R_1) - G(R_2)$.

\paragraph{Contributions.}
\begin{enumerate}
  \itemsep0em
  \item \textbf{Formal.} We define $G(R)$ and identify partition score
    as an oracle-assisted upper bound rather than a system score. Two
    accounting lemmas (Lemmas~\ref{lem:gap-identity},~\ref{lem:absorption})
    support a per-instance margin--regret stability theorem
    (Theorem~\ref{thm:margin-regret}) and a sharp partition-only
    identification theorem (Theorem~\ref{thm:identification}) showing
    when partition reports cannot certify the deployable winner.
  \item \textbf{Empirical.} On five public AS benchmarks we separate
    \emph{decomp-vs-flat} deployment decisions, where the safety
    condition fails in $4/10$ comparisons and B2/B4 have sign-changing
    point estimates, from
    \emph{decomp-vs-decomp} outer-selector comparisons, where pairwise
    and multiclass are nearly indistinguishable
    ($|\Apart| < 5\times 10^{-3}$). Margin--regret violation rates
    ($0.23$--$0.56$) and identification intervals (all $15$ pairs
    cross zero) corroborate the theorems.
  \item \textbf{Practical.} A training-side gap-correction diagnostic
    recovers the point-estimate deployable sign on all four
    sign-changing cells; a deployment-aware family selector shrinks
    $G(R)$ on every
    benchmark and improves $\Setoe$ on B1--B4, with a small PROTEUS
    trade-off. We give a
    checklist requiring $\Spart$, $\Setoe$, $G(R)$, the safety
    condition, and the identification interval to be reported together
    (\S\ref{sec:guidance}).
\end{enumerate}

\section{Related Work}
\label{sec:related}

\paragraph{Oracle baselines and VBS.} Algorithm selection has long used
oracle baselines such as the Virtual Best Solver (VBS) as an upper bound
on portfolio performance \citep{xu2008satzilla}. Frameworks such as
AutoFolio \citep{lindauer2015autofolio} and standardised AS scenarios in
ASlib \citep{bischl2016aslib} explicitly compare deployable selectors to
this VBS, and the AS competitions normalise scores by the SBS--VBS
interval \citep{lindauer2019competitions}. \citet{cameron2016vbs} further
document that VBS evaluation itself can be optimistically biased when
solvers are randomised. Oracle-style scores beyond the full-portfolio VBS
are also routinely reported. \emph{Selected-portfolio} VBS appears in
$k$-portfolio studies \citep{bach2022kportfolios} and in portfolio
selection for automated algorithm selection
\citep{kostovska2023psaas}; \emph{subportfolio} virtual-best scores are
reported alongside the deployable system in hierarchical solver
portfolios such as Proteus \citep{hurley2014proteus}, whose four
families correspond to a CSP-native solver branch and three SAT
encodings of the input. Encoding-level selection has its own
\emph{virtual-best encoding} as a standard upper bound
\citep{stojadinovic2014mesat,ulricholtean2022encodings,ulricholtean2023learning},
and tabular benchmarking literature reports
\emph{best-in-family} summaries -- best deep model versus best gradient
boosting model -- as the primary comparison object
\citep{mcelfresh2023tabzilla,shmuel2025tabular}. A literature audit of
twelve such reporting objects, the claim each row supports, and the strength
of that support is given in Appendix~\ref{app:literature-audit};
\citet{shmuel2025tabular} is a representative tabular instance in which
the best-of-DL versus best-of-TE comparison is the headline conclusion
and no within-family deployable selector is specified.

Closest in spirit, \citet{tornede2023metalevel} define an AS-oracle
over algorithm selectors and show that lifting the oracle to that
meta level can degrade oracle performance; our setting differs in
that the hidden oracle is over algorithms inside an already-chosen
family rather than over complete selectors, so the relevant loss is
the within-family \dfgap{} $G(R)$.

\paragraph{Methodological critiques.}
\citet{petelin2025pitfalls} identify two pitfalls in feature-based AS
benchmarking --- leave-instance-out (LIO) splits and scale-sensitive
performance targets --- both orthogonal to our concern: LIO controls
\emph{which} test instances are held out and scale sensitivity
controls the numeric \emph{target}, whereas the \dfgap{} concerns
\emph{which system} is evaluated. Unlike split or target-scale
effects, the \dfgap{} changes the evaluated object itself: an
oracle-assisted partition score replaces one stage of the deployable
selector.

\paragraph{Decomposition in algorithm selection and classification.}
Pairwise classification \citep{sun2013pairwise,galar2011overview},
multiclass selectors, and family-level decompositions are well-studied as
\emph{training} strategies, often analysed within the
error-correcting-output-code framework \citep{allwein2000reducing}. We
study how such decompositions are \emph{evaluated} once an inner
within-family decision is left to a learned selector at deployment time.

\section{Setup}
\label{sec:setup}

Let $x \in \mathcal{X}$ denote a problem instance (e.g.\ a dataset or a
SAT formula), $\mathcal{A}$ a finite algorithm pool, and
$\Pi = \{F_1,\dots,F_K\}$ a partition of $\mathcal{A}$ into algorithm
families ($F_i \cap F_j = \emptyset$ for $i \neq j$ and
$\bigcup_k F_k = \mathcal{A}$). Let $u(x,a) \in \mathbb{R}$ denote the
utility of algorithm $a$ on instance $x$; higher is better. Throughout the
paper we use a utility convention for both score and advantage. Benchmarks
whose native metric is a regret or a runtime are mapped to utilities by a
\emph{per-instance} monotone transformation (Appendix~\ref{app:transform}),
which preserves the within-instance ordering of algorithms; all reported
magnitudes and advantages are therefore interpreted on the resulting
normalised utility scale.

A \emph{decomposed selector} $R$ consists of a family-level selector
$h_R : \mathcal{X} \rightarrow \{1,\dots,K\}$ and, for each family $k$, a
within-family selector $g_{R,k} : \mathcal{X} \rightarrow F_k$. Let
$\famhat(x) = h_R(x)$ denote the predicted family, and define the
family-$k$ oracle-best algorithm
\begin{equation}
  a_k^{*}(x) \;=\; \arg\max_{a \in F_k} u(x,a).
\end{equation}

\paragraph{Partition score.} The \emph{partition-level} evaluation grants
an oracle choice within the predicted family:
\begin{equation}
  \Spart(R) \;=\; \mathbb{E}_{x}\bigl[\, u\bigl(x,\, a_{\famhat(x)}^{*}(x) \bigr) \,\bigr].
  \label{eq:spart}
\end{equation}

\paragraph{End-to-end score.} The \emph{end-to-end} evaluation deploys the
within-family selector that is actually available:
\begin{equation}
  \Setoe(R) \;=\; \mathbb{E}_{x}\bigl[\, u\bigl(x,\, g_{R,\famhat(x)}(x) \bigr) \,\bigr].
  \label{eq:se2e}
\end{equation}

\paragraph{Deployment-fidelity gap.} $G(R) := \Spart(R) - \Setoe(R)$
measures how much partition-level evaluation overstates the deployable
system. For two pipelines $R_1, R_2$ evaluated on the same test
distribution, the partition and end-to-end advantages are
$\Apart := \Spart(R_1) - \Spart(R_2)$ and
$\Aetoe := \Setoe(R_1) - \Setoe(R_2)$.

\section{From Partition Scores to Deployment Stability}
\label{sec:algebra}

This section relates the four estimands of \S\ref{sec:setup} in three
steps. \S\ref{sec:lemmas} records two accounting lemmas: $G(R)$ equals
within-family regret, and advantage absorption equals the difference of
gaps. \S\ref{sec:margin-regret} states a margin--regret stability
theorem characterising, instance by instance, when the family chosen at
partition time is also deployment-optimal. \S\ref{sec:identification}
shows that even with selected-family utility ranges added to a
partition-level report, the deployable advantage between two pipelines
is identifiable only within an interval, and that when the interval
strictly crosses zero the report cannot certify a deployable winner.

\subsection{Accounting lemmas}
\label{sec:lemmas}

\begin{lemma}[Gap identity]
\label{lem:gap-identity}
For any decomposed selector $R$,
\begin{equation}
  G(R) \;=\; \mathbb{E}_{x}\Bigl[\, u\bigl(x,\, a_{\famhat(x)}^{*}(x) \bigr) - u\bigl(x,\, g_{R,\famhat(x)}(x) \bigr) \,\Bigr] \;\geq\; 0,
\end{equation}
with equality if and only if the within-family selector chooses an
oracle-optimal algorithm inside the predicted family almost surely.
\end{lemma}

\begin{proof}[Proof sketch]
Substitute~\eqref{eq:spart},~\eqref{eq:se2e} into $G(R)$ and use
nonnegativity of $u(x, a_{\famhat(x)}^*(x)) - u(x, g_{R,\famhat(x)}(x))$
on each instance. Full proof in Appendix~\ref{app:proofs}.
\end{proof}

\begin{lemma}[Advantage absorption]
\label{lem:absorption}
For any two decomposed selectors $R_1, R_2$,
\begin{equation}
  \Aetoe \;=\; \Apart \;-\; \bigl[\, G(R_1) - G(R_2) \,\bigr],
  \label{eq:absorption}
\end{equation}
or equivalently $\Apart - \Aetoe = G(R_1) - G(R_2)$.
\end{lemma}

\begin{proof}[Proof sketch]
$\Setoe(R) = \Spart(R) - G(R)$ by definition; expand $\Aetoe = \Setoe(R_1) - \Setoe(R_2)$. Full proof in Appendix~\ref{app:proofs}.
\end{proof}

The lemmas have two immediate consequences we use throughout the paper.
First (Corollary~\ref{cor:family-vs-within}), $G(R)$ measures
within-family suboptimality: it is zero exactly when the within-family
selector is family-oracle, regardless of whether $\famhat(x)$ is the
globally optimal family. Second (Corollary~\ref{cor:flat-safety}), the
partition-level claim ``$R$ beats $F$'' transfers to end-to-end
evaluation iff $\Apart(R, F) > \Delta G(R, F)$, where
$\Delta G := G(R) - G(F)$; for an ideal flat baseline ($G(F) = 0$) this
reduces to $\Apart > G(R)$.

\begin{corollary}[Family error vs.\ within-family error]
\label{cor:family-vs-within}
$G(R)$ depends on $\famhat(x)$, but not on whether $\famhat(x)$ is
globally optimal: even if $R$ chooses the wrong family, the gap is zero
whenever the within-family selector picks an oracle-best algorithm in
that wrong family; conversely, even with the globally optimal family
selected, $G(R) > 0$ whenever the within-family selector deviates from
$a_{\famhat(x)}^{*}(x)$.
\end{corollary}

\begin{corollary}[General deployment-safety condition]
\label{cor:flat-safety}
For any two decomposed selectors $R$ and $F$ (the latter typically a
flat baseline), the partition-level conclusion ``$R$ beats $F$''
transfers to end-to-end evaluation if and only if
$\Apart(R, F) > \Delta G(R, F)$, where $\Delta G := G(R) - G(F)$. When
$G(F) = 0$ this reduces to $\Apart > G(R)$.
\end{corollary}

\begin{proof}[Proof sketch]
Apply Lemma~\ref{lem:absorption}: $\Aetoe > 0 \iff \Apart > \Delta G$.
\end{proof}

Corollary~\ref{cor:flat-safety} turns the apparently tautological
identity ($G(F) = 0$ for ideal flat baselines) into a falsifiable
empirical safety condition; \S\ref{sec:empirical} documents how often it
fails. The general $\Delta G$ form is kept because it covers arbitrary
reference systems, while our singleton flat controls enforce the
invariant $G(F)=0$ by applying the same missing-value fallback to
partition and end-to-end singleton scores.

\subsection{Margin--regret stability of family selection}
\label{sec:margin-regret}

The lemmas describe \emph{aggregate} accounting. To say when partition
conclusions transfer at the level of an individual instance, we
decompose deployable utility into an oracle family margin and a
within-family regret, then read off when the family selected at
partition time is deployment-optimal.

\paragraph{Per-family quantities.}
For each instance $x$ and each family $k$ define
\begin{equation}
  m_k(x) \;=\; \max_{a \in F_k} u(x, a),
  \qquad
  v_k(x) \;=\; u(x,\, g_{R,k}(x)),
  \qquad
  r_k(x) \;=\; m_k(x) - v_k(x),
\end{equation}
so $v_k(x) = m_k(x) - r_k(x)$. Here $m_k(x)$ is family $k$'s oracle
utility on $x$, $v_k(x)$ is the deployable utility realised by $R$'s
within-family selector $g_{R,k}$ inside family $k$, and
$r_k(x) \ge 0$ is the within-family regret. Writing $s = \famhat(x)$ for
the family that the partition-time selector actually chooses, the
deployable utility of $R$ at $x$ is exactly $v_s(x)$.

\begin{theorem}[Margin--regret deployment stability]
\label{thm:margin-regret}
Fix an instance $x$ and a family selector $\famhat$, and write
$s = \famhat(x)$. The deployment regret of selecting family $s$
relative to the best deployable family on $x$ is
\begin{equation}
  \mathrm{Reg}_{\mathrm{dep}}(\famhat;\, x)
  \;=\; \max_{q \in [K]} v_q(x) \;-\; v_s(x)
  \;=\; \max_{q \in [K]}\bigl[\,(m_q(x) - m_s(x)) - (r_q(x) - r_s(x))\,\bigr].
\end{equation}
In particular, $s$ is deployment-optimal at $x$ if and only if
\begin{equation}
  m_s(x) - m_q(x) \;\ge\; r_s(x) - r_q(x)
  \quad\text{for every competing family } q \in [K].
  \label{eq:stability-condition}
\end{equation}
\end{theorem}

\begin{proof}[Proof sketch]
Substitute $v_k = m_k - r_k$ and maximise over $q$;
non-negativity of regret yields~\eqref{eq:stability-condition}.
Full proof in Appendix~\ref{app:proofs}.
\end{proof}

\paragraph{Reading the theorem.} The decomposition makes explicit what
the aggregate identity $\Aetoe = \Apart - \Delta G$ hides. Deployment
failure of family $s$ relative to family $q$ is controlled by a
\emph{differential} within-family regret $r_s - r_q$ measured against an
oracle family margin $m_s - m_q$: $s$ stays deployment-optimal as long
as it has a larger oracle margin than the gap between its within-family
selector and family $q$'s. Two practical consequences. First, the
partition-optimal family $p(x) \in \arg\max_k m_k(x)$ remains
deployment-optimal at $x$ iff $m_p - m_q \ge r_p - r_q$ for every
$q \ne p$; otherwise the partition-time choice $p(x)$ is dominated in
deployment by some competitor with smaller oracle margin but smaller
within-family regret too. Second, the Bayes-optimal deployable outer
target is $v_k(x) = m_k(x) - r_k(x)$, not $m_k(x)$. Training a family
selector against $m_k$ optimises family oracle potential; training
against (a cross-fitted estimator of) $v_k$ optimises the quantity that
determines deployment-optimal family choice. \S\ref{sec:guidance}
returns to this in describing our deployment-aware family selector.

A synthetic toy example with two singleton-conflicting decompositions
illustrates the general deployment-safety condition of
Corollary~\ref{cor:flat-safety}: $\Apart \approx +0.20$ but
$\Aetoe \approx -0.24$ because the differential gap exceeds the
partition advantage ($\rho \approx 2.2$). The per-instance
margin--regret condition explains the within-pipeline mechanism; the
cross-pipeline sign change itself is governed by
Corollary~\ref{cor:flat-safety} (Appendix~\ref{app:toy}).

\subsection{Why partition-level reports cannot certify deployable winners}
\label{sec:identification}

Theorem~\ref{thm:margin-regret} characterises one pipeline's deployment
stability. We now turn to comparisons of two pipelines $R_1, R_2$ and
ask: from a partition-level report, can the reader certify the
deployable winner? The answer is negative; even with selected-family
utility ranges added to the report, the deployable advantage $\Aetoe$
is identifiable only within an interval that may cross zero.

\paragraph{Selector classes and gap envelopes.}
For pipeline $R$ with fixed family selector $\famhat$, let
$\mathcal{G}_R$ denote a class of admissible within-family selectors
$g$, each producing a gap $G_g(R) \ge 0$ via Lemma~\ref{lem:gap-identity}.
Define the lower and upper gap envelopes
\begin{equation}
  \underline{G}(R) \;:=\; \inf_{g \in \mathcal{G}_R}\, G_g(R),
  \qquad
  \overline{G}(R) \;:=\; \sup_{g \in \mathcal{G}_R}\, G_g(R).
\end{equation}
The unrestricted family-respecting class (every measurable $g$ with
$g(x) \in F_{\famhat(x)}$) gives $\underline{G}(R) = 0$ (oracle
within-family) and
$\overline{G}(R) = W_R := \mathbb{E}_x\bigl[\,
\max_{a \in F_{\famhat(x)}} u(x,a) - \min_{a \in F_{\famhat(x)}} u(x,a)
\,\bigr]$ (worst within-family choice attainable as a measurable
selector).

\begin{theorem}[Partition-only identification]
\label{thm:identification}
Let $R_1, R_2$ be two decomposed selectors and let
$\mathcal{G}_{R_i}$ be admissible within-family selector classes.
Suppose only $\Spart(R_1)$, $\Spart(R_2)$, and the envelopes
$\underline{G}(R_i), \overline{G}(R_i)$ are reported. Then
\begin{equation}
  \Aetoe \;\in\;
  \bigl[\,
    \Apart - \overline{G}(R_1) + \underline{G}(R_2),
    \;\;
    \Apart - \underline{G}(R_1) + \overline{G}(R_2)
  \,\bigr],
  \label{eq:identification-interval}
\end{equation}
and if the gap envelopes $\underline{G}(R_i), \overline{G}(R_i)$ are
attained by admissible selectors, the interval is sharp; otherwise the
endpoints are approached arbitrarily closely by some choice of
within-family selectors $g_1 \in \mathcal{G}_{R_1}, g_2 \in \mathcal{G}_{R_2}$.
Under the unrestricted family-respecting class, the envelopes are
attained by the within-family oracle and the within-family worst-case
selector, so the interval $[\Apart - W_{R_1},\, \Apart + W_{R_2}]$ is
sharp without further assumption.
\end{theorem}

\begin{proof}[Proof sketch]
Apply Lemma~\ref{lem:absorption} to $(g_1, g_2)$ and vary independently
over $\mathcal{G}_{R_i}$. Under the unrestricted class, $\underline{G}=0$
is attained by the within-family oracle and $\overline{G}=W_R$ by the
pointwise within-family worst-case selector. Full proof in
Appendix~\ref{app:proofs}.
\end{proof}

\begin{corollary}[No partition-only sign certificate]
\label{cor:no-sign-cert}
If the interval in~\eqref{eq:identification-interval} strictly crosses zero,
then the partition-level report (with envelopes) does not certify the
sign of $\Aetoe$: there exist admissible $(g_1, g_2)$ and
$(g_1', g_2')$ realising the same partition-level scores and envelopes
but $\Aetoe(g_1, g_2)$ and $\Aetoe(g_1', g_2')$ of opposite sign.
Conversely, if the interval is strictly positive (resp.\ strictly
negative), then $\Aetoe$ has the same sign as the interval for every
admissible $(g_1, g_2)$.
\end{corollary}

\paragraph{Scope qualifier.} Theorem~\ref{thm:identification} is a
statement about \emph{partition-level reports}, not about the underlying
benchmark. It does not say the deployable winner is unknowable: once a
concrete within-family selector is fixed and evaluated, $\Aetoe$ is
just a number. What it does say is that the partition-level report,
even augmented with selected-family utility ranges, is insufficient as
a \emph{system-performance certificate} when its identification
interval strictly crosses zero. Equivalently: full transfer of the partition
conclusion requires either (i) reporting $\Setoe$ directly, which
collapses the interval to a point, or (ii) demonstrating empirically
(e.g.\ via the margin--regret diagnostic of
Theorem~\ref{thm:margin-regret} and \S\ref{sec:rq3}) that within-family
regret is small relative to the oracle margin on the relevant instances.

\paragraph{Reversal taxonomy.} Define the absorption ratio
$\rho := 1 - \Aetoe / \Apart = (G(R_1) - G(R_2))/\Apart$, well-defined
when $|\Apart| > 0$. We label transfer ($\rho \approx 0$), partial
absorption ($0 < \rho < 1$), extreme absorption ($\rho \approx 1$),
reversal ($\rho > 1$, $\Apart > 0$), and opposite reversal ($\Apart < 0$,
$\Aetoe > 0$); the full taxonomy table is in Appendix~\ref{app:taxonomy}.

\paragraph{Implication for empirical studies.} The lemmas turn
``\emph{does within-family suboptimality explain the gap?}'' and
``\emph{is partition advantage equal to deployable advantage?}'' from
statistical hypotheses into definitional identities. The two theorems
turn the questions ``\emph{when is partition-optimal family choice
deployment-stable?}'' and ``\emph{can a deployable winner be certified
from partition reports?}'' into testable per-instance and per-pair
conditions. \S\ref{sec:rq3} reports two diagnostics that operationalise
these conditions: a per-instance margin--regret violation rate
(Theorem~\ref{thm:margin-regret}) and a partition-only identification
interval together with its crosses-zero count
(Theorem~\ref{thm:identification}). Pointwise computation requires a
common valid-instance set per comparison; the empirical residual of
Lemma~\ref{lem:absorption} is below $10^{-15}$ on every benchmark
(Appendix~\ref{app:nan-handling}).

\section{Empirical Results}
\label{sec:empirical}

\subsection{Benchmarks and protocol}
\label{sec:protocol}

Table~\ref{tab:benchmarks} summarises the five benchmarks: three
tabular AutoML AS scenarios (B1--B3) and two ASlib
combinatorial-search scenarios (B4--B5). We use fixed matrix-style
resources because our estimands require paired algorithm-by-instance
utilities under a stable protocol; living benchmarks such as TabArena
\citep{erickson2025tabarena} are complementary. Family selectors are
multiclass RF, pairwise one-vs-one RF with soft voting, or a synthetic
singleton-family flat baseline. The default within-family selector is
a per-algorithm gradient-boosted regressor; \S\ref{sec:robustness}
sweeps three more classes. We run $5$ seeds $\times$ $5$-fold
stratified CV and convert performance to per-instance utilities in
$[0,1]$ (Appendix~\ref{app:transform}). B4--B5 family mappings follow
solver-competition reports \citep{li2009maxsat,hurley2014proteus} and
appear verbatim in Appendix~\ref{app:family-mapping}. PROTEUS-2014 is
partitioned semantically into CSP-native solvers plus three SAT
encoding families; its primary feature policy exposes only the ASlib
\texttt{csp} step ($36$ features), with an all-feature sensitivity in
Appendix~\ref{app:proteus-features}.

All diagnostics and CIs use the test instance as the unit: repeated
seed/fold predictions are aggregated per instance before
instance-cluster bootstrap (Appendix~\ref{app:nan-handling}).

\begin{table}[t]
\small
\centering
\caption{Benchmark registry. Native metrics are converted to
per-instance utility (App.~\ref{app:transform}). B5 retains the common
valid-instance set after dropping all-timeout rows from $4021$ raw
PROTEUS-2014 instances.}
\label{tab:benchmarks}
\begin{tabular}{llrrrll}
\toprule
ID & Benchmark        & \#inst.\ & \#alg.\ & \#fam.\ & Domain & Native metric \\
\midrule
B1 & TabZilla         & 133 & 22 & 4 & tabular ML            & accuracy \\
B2 & TabRepo (cls.)   & 202 & 11 & 4 & tabular ML            & 1$-$error rate \\
B3 & TALENT (cls.)    & 200 & 34 & 4 & tabular ML            & accuracy \\
B4 & MAXSAT-PMS-2016  & 556 & 19 & 4 & combinatorial search  & PAR10 \\
B5 & PROTEUS-2014     & 3565 & 22 & 4 & combinatorial CSP/SAT & runtime \\
\bottomrule
\end{tabular}
\end{table}

\subsection{RQ1: Do partition scores overstate deployable performance?}
\label{sec:rq1}

\begin{figure}[t]
\centering
\includegraphics[width=0.95\linewidth]{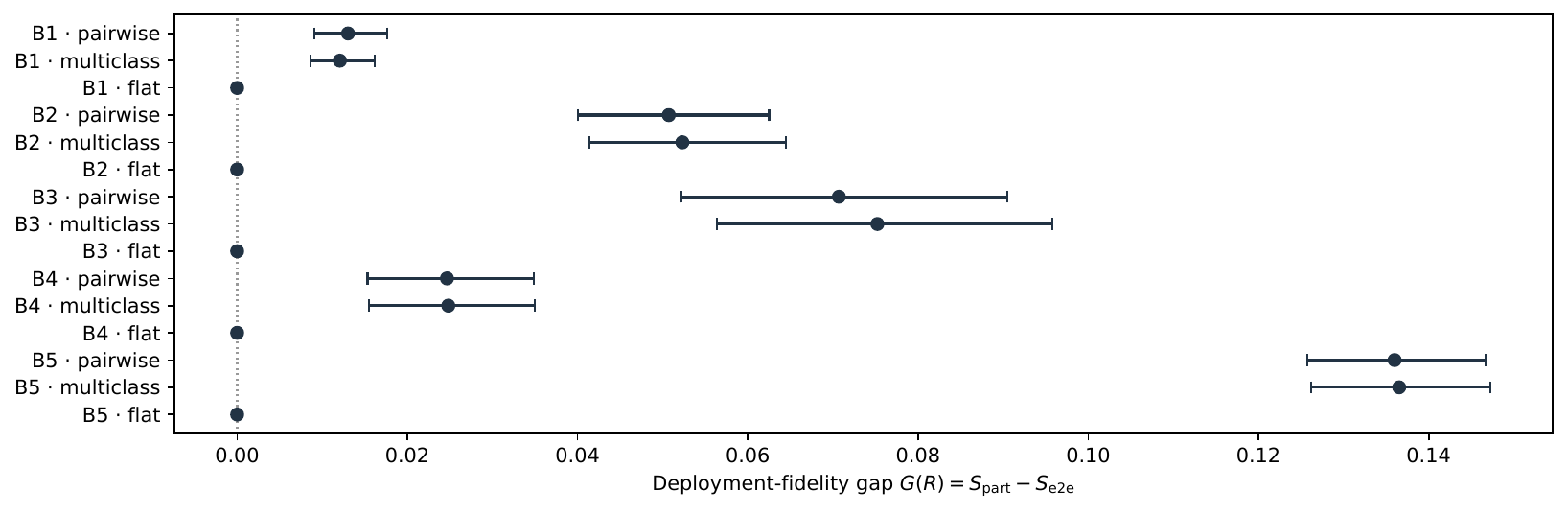}
\caption{Deployment-fidelity gap $G(R) = \Spart - \Setoe$ for every
(benchmark, decomposition) cell. Bars show $95\%$ percentile-bootstrap
confidence intervals on the per-instance gap. All decomposed pipelines
exhibit positive gaps; flat baselines are included as controls. The
singleton-flat invariant enforces $G(F)=0$ for every flat baseline by
using identical missing-value fallback on partition and end-to-end
singleton scores.}
\label{fig:forest}
\end{figure}

Figure~\ref{fig:forest} plots $G(R)$ across all $15$ cells. Flat
baselines are controls; the main claim concerns pairwise and
multiclass decomposed pipelines, all with positive gaps and CIs
excluding zero. Gaps range from $0.012$ on TabZilla (B1) to
approximately $0.13$ on PROTEUS-2014 (B5). Paired Wilcoxon tests yield
$p < 10^{-15}$ on every
decomposed cell; full $p$-values are in Appendix~\ref{app:full-tables}.

\subsection{RQ2: When do partition-level conclusions transfer?}
\label{sec:rq2}

For each benchmark we compare pairwise OVO vs.\ flat, multiclass
softmax vs.\ flat, and pairwise vs.\ multiclass. The $15$ pairs split
into \emph{decomp-vs-flat} deployment decisions (should we decompose?)
and \emph{decomp-vs-decomp} outer-selector comparisons; the regimes
behave differently, so we report them separately.

\subsubsection{Decomp-vs-flat: is decomposition deployment-safe?}
\label{sec:rq2-flat}

Among the $10$ decomp-vs-flat comparisons in Table~\ref{tab:transfer},
$4$ have sign-changing point estimates: B2 and B4 fail the safety condition
$\Apart > \Delta G$ (equal to $\Apart > G(R)$ here because every flat
baseline has $G(F)=0$), so partition scores prefer decomposition while
end-to-end point estimates prefer flat selection. One B2 cell and the
two B4 cells have small end-to-end margins (B2 pw: $-0.0037$; B4:
$-0.0047$ and $-0.0031$), so we treat them as small-margin sign changes
rather than large-margin reversals. B1 is a
high-absorption case:
the partition advantage remains positive end-to-end but shrinks to
$1$--$2$ utility points ($\rho \approx 0.84$--$0.91$). B3 and
PROTEUS-2014 (B5) are partially absorbed; B5's $33$ pp partition
advantage shrinks to $20$ pp end-to-end ($\rho \approx 0.41$).
Supporting plots are in Appendix~\ref{app:full-tables}.

\subsubsection{Decomp-vs-decomp: do outer selectors differ?}
\label{sec:rq2-outer}

The remaining $5$ pairs (pairwise vs.\ multiclass) all have
$|\Apart| < 5\times 10^{-3}$ (Appendix~\ref{app:full-tables}); the
outer-selector choice is empirically secondary to whether the score
includes the deployable within-family selector.

\begin{table}[t]
\small
\centering
\caption{Decomp-vs-flat advantage transfer. ``pw'' = pairwise OVO,
``mc'' = multiclass softmax. $\rho > 1$ marks sign change relative to
the partition-level advantage; small end-to-end margins should be read
as point estimates rather than large-margin reversals. All quantities
are recomputed on the four-way intersection of finite values for the
specific pair, so the gap values in this table need not match the
per-cell marginal gaps in Table~\ref{tab:exp2-full}; on each pair,
$\Apart - \Aetoe = G(R_1) - G(R_2)$ holds exactly
(Lemma~\ref{lem:absorption}).
Pairwise vs.\ multiclass comparisons all have $|\Apart| < 5\times 10^{-3}$
and are reported in Appendix~\ref{app:full-tables} (Table~\ref{tab:exp3-full}).}
\label{tab:transfer}
\begin{tabular}{llrrrrrl}
\toprule
ID & Pair & $\Apart$ & $\Aetoe$ & $G(R_1)$ & $G(R_2)$ & $\rho$ & Regime \\
\midrule
B1 & mc vs flat & $+0.0145$ & $+0.0024$ & $0.0121$ & $0.0000$ & $0.835$ & high absorption \\
B1 & pw vs flat & $+0.0143$ & $+0.0013$ & $0.0130$ & $0.0000$ & $0.912$ & high absorption \\
B2 & mc vs flat & $+0.0446$ & $-0.0077$ & $0.0523$ & $0.0000$ & $1.172$ & \textbf{sign change} \\
B2 & pw vs flat & $+0.0471$ & $-0.0037$ & $0.0507$ & $0.0000$ & $1.078$ & \textbf{small sign change} \\
B3 & mc vs flat & $+0.1069$ & $+0.0316$ & $0.0752$ & $0.0000$ & $0.704$ & partial absorption \\
B3 & pw vs flat & $+0.1079$ & $+0.0372$ & $0.0707$ & $0.0000$ & $0.655$ & partial absorption \\
B4 & mc vs flat & $+0.0202$ & $-0.0047$ & $0.0248$ & $0.0000$ & $1.231$ & \textbf{small sign change} \\
B4 & pw vs flat & $+0.0216$ & $-0.0031$ & $0.0247$ & $0.0000$ & $1.142$ & \textbf{small sign change} \\
B5 & mc vs flat & $+0.3285$ & $+0.1919$ & $0.1365$ & $0.0000$ & $0.416$ & partial absorption \\
B5 & pw vs flat & $+0.3333$ & $+0.1973$ & $0.1360$ & $0.0000$ & $0.408$ & partial absorption \\
\bottomrule
\end{tabular}
\end{table}

\paragraph{B5 robustness checks.}
Because PROTEUS-2014 gives the largest gap, we test two design choices.
First, giving the selector all $198$ features (including
encoding-conditional features unavailable before choosing an encoding)
improves deployment but leaves $G>0$: pairwise $0.134 \to 0.100$,
multiclass $0.135 \to 0.098$. Second, shortening the runtime cutoff
from $3600$\,s to $1800$\,s leaves the gap essentially unchanged
(pairwise $0.134 \to 0.134$, multiclass $0.135 \to 0.136$). Thus the
B5 headline effect is not an artefact of feature availability or
timeout convention; full numbers are in
Appendix~\ref{app:proteus-features}.

\subsection{RQ3: Theorem 1/2 diagnostics}
\label{sec:rq3}

Lemma~\ref{lem:gap-identity}/\ref{lem:absorption} residuals are at
machine precision (Appendix~\ref{app:identity-residuals}), so we focus
on the diagnostics induced by Theorems~\ref{thm:margin-regret}
and~\ref{thm:identification}.

\paragraph{Theorem 1 diagnostic: per-instance margin--regret stability.}
For each instance $x$ we form
$\mathrm{viol}(x) = \mathbf{1}\{\exists q \ne s : r_s(x) - r_q(x) >
m_s(x) - m_q(x)\}$ for the selected family $s=\famhat(x)$. A violation
means a competitor has higher deployable utility because its lower
oracle margin is offset by lower within-family regret. The existential
rate ranges $0.23$--$0.56$ for pairwise and multiclass pipelines
(Table~\ref{tab:margin-regret-rates}); flat-baseline rates are larger
($0.49$--$0.67$) because they count any singleton beating the
single-best-on-train algorithm.

\paragraph{Theorem 2 diagnostic: partition-only identification interval.}
For each comparison we report
$[\Apart - W_{R_1}, \Apart + W_{R_2}]$, with $W_R$ the average
selected-family utility range. All $15$ intervals strictly cross zero,
including B5 ($\Apart = +0.33$, $W_R \approx 0.75$), so partition
reports plus family ranges do \emph{not} certify the deployable
winner. This interval is a conservative certificate, not a predictor;
direct $\Setoe$ reporting remains necessary
(Table~\ref{tab:interval-summary}, Appendix~\ref{app:full-tables}).

\subsection{RQ4: Robustness to within-family selector class}
\label{sec:robustness}

Across within-family selector classes, $G(R)$ varies $1.28\times$ to
$2.23\times$ by benchmark, while family-selector class changes
$\Spart$ and $G$ negligibly. In the main B4--B5 sweep, the largest
PROTEUS-2014 gap is $G=0.211$, and even the best class leaves positive
B4--B5 gaps
(Appendix~\ref{app:exp5-b4b5}); reports must specify the
within-family selector.

\paragraph{Inner-selector stress test.} Sweeping RF, XGBoost,
LightGBM, MLP, and tuned GBDT within-family regressors does not
collapse $G(R)$: the best realistic selector preserves $87$--$100\%$
of the default on four of five benchmarks (RF on B1 is the lone
sub-$80\%$ outlier). A targeted rerun on B2/B4 with RF and LightGBM
within-family selectors keeps $G(R)$ positive in every cell
($0.046$--$0.049$ on B2 and $0.024$--$0.027$ on B4). The strict sign
of the small B2/B4 end-to-end margins is selector-sensitive, but no
stronger selector produces a reliable positive decomp-vs-flat
advantage. B1 remains a high-absorption case, and a B5 rerun with RF
keeps PROTEUS-2014 at $\rho \approx 0.41$
(Appendix~\ref{app:spectrum}). The gap is not a weak-inner-selector
artifact, although exhaustive HPO or task-specific meta-learners could
still reduce it and must be evaluated end-to-end.

\section{Correction and Practical Guidance}
\label{sec:guidance}

\paragraph{Gap-corrected reporting.} Estimate $G$ on held-out folds
and correct partition advantage by the gap difference,
$\widehat{A}_{\mathrm{corr}} = \widehat{A}_{\mathrm{part}}^{\mathrm{test}} - [\widehat{G}^{\mathrm{val}}(R_1) - \widehat{G}^{\mathrm{val}}(R_2)]$.
Using validation folds drawn only from the outer training split,
$\widehat{A}_{\mathrm{corr}}$ reduces MAE by $25$--$88\%$ on B2--B5
decomp-vs-flat cells and recovers the point-estimate deployable sign on
all four sign-changing cells (B2/B4). On B1, where the true end-to-end
advantage is small but positive, correction can over-subtract and add
variance. Split-sensitivity reruns on B2/B4 with $3$, $5$, and $10$
training-side validation folds, and two additional validation split
seeds at the $5$-fold setting, recover the same four point-estimate
signs; MAE reduction is weaker on the most granular validation split,
especially for B4 (Appendix~\ref{app:correction}). The diagnostic also
assumes the training-side validation split is representative of the
outer test distribution; under covariate or task-distribution shift it
may under- or over-correct. It does not replace direct end-to-end
evaluation.

\paragraph{Deployment-aware family selector.} Training the outer
selector on cross-fitted deployable utility instead of
$\max_{a \in F_k}u$ shrinks $G(R)$ on every benchmark and improves
$\Setoe$ on B1--B4; on B5 it shrinks $G$ by $0.008$ with a small
$\Setoe$ trade-off ($-0.005$, Appendix~\ref{app:deployment-aware}).

\paragraph{Reporting checklist.} A decomposed-AS report should pair
$\Spart, \Setoe, G(R)$ with $\Apart, \Aetoe, \rho$ and the
identification interval, and document the within-family selector,
family mapping, feature policy, and common valid-instance set
(Appendix~\ref{app:guidance}).

\section{Limitations and Conclusion}
\label{sec:limitations}

$G(R)$ is mapping-dependent (Appendix~\ref{app:family-sensitivity})
and pipeline-specific; B5's headline effect survives the
feature-policy and cutoff sensitivities of \S\ref{sec:rq2-flat}.
Continuous black-box optimisation \citep{petelin2025pitfalls} is out
of scope. Stronger within-family models can move small end-to-end
margins (Appendix~\ref{app:spectrum}), and exhaustive HPO or
domain-specific meta-learners may shrink $G(R)$ further; such
improvements must still be reported end-to-end. The validation
gap-correction diagnostic assumes training-side validation is
representative of the test distribution and is not evaluated under
explicit covariate shift. Partition scores upper-bound family quality
but are not system scores: the safety condition fails on $4/10$ cells and every
identification interval strictly crosses zero, so $\Spart, \Setoe, G(R)$, and
the interval should be reported together.
\textbf{Broader impact.} This evaluation-methodology study (no
human-subject data, no new deployed system) reduces misleading
AS/AutoML reporting that can overstate deployable performance; its
main negative impact is over-certification --- our diagnostics flag
when end-to-end evaluation is needed, not deployment-safety
guarantees, and should accompany $\Setoe$ and task-specific risk
assessment.

\clearpage
\bibliographystyle{plainnat}
\bibliography{paper}

\clearpage
\appendix
\section*{Appendix}

\noindent The appendix contains: (A) the per-instance utility
transform; (B) full algorithm-family mappings for all five benchmarks;
(C) the synthetic toy example used in \S\ref{sec:algebra}; (D) the
reversal taxonomy table; (E) the full reporting checklist and decision
framework; (F) full proofs of Lemmas 1--2 and Theorems 1--2; (G)
common-instance evaluation and identity-residual analysis; (H)
pseudocode for the decomposed pipeline evaluator; (I) gap-corrected
reporting; (J) family-mapping sensitivity (B4 only; B5 reports the
primary semantic partition); (K) deployment-aware family selector;
(L) PROTEUS-2014 feature-policy and cutoff sensitivities; (M)
metric-transform sensitivity on B4--B5; (N) the within-family selector
robustness sweep with both the per-benchmark figure and the full
$24$-row B1--B3 table; (O) full numerical tables for the gap and
advantage-transfer analyses; (P) statistical methodology; (Q) compute
resources and reproducibility instructions; (R) a literature audit of
oracle-style reporting in algorithm selection, solver portfolios,
encoding selection, and tabular benchmarking; (S) full per-cell
results of the inner-selector stress test of \S\ref{sec:robustness}.

\section{Utility transformation}
\label{app:transform}

For B1--B3 the native metric is an accuracy-like value already in
$[0,1]$ and higher-is-better; we use it directly as utility. For B4
(PAR10) and B5 (runtime; PAR10 reconstructed from \texttt{runstatus}),
we apply a per-instance min-max transformation
\begin{equation}
  u(x, a) \;=\; 1 \;-\; \mathrm{clip}\!\left(\frac{p(x,a) - p_{\min}(x)}{p_{\max}(x) - p_{\min}(x)},\; 0,\; 1\right),
\end{equation}
where $p(x, a)$ is the native lower-is-better measurement (PAR10 in
seconds). Instances on which all algorithms tie at the timeout penalty
are dropped (these carry no AS signal). More generally, the transform
is applied only when $p_{\max}(x) > p_{\min}(x)$; if
$p_{\max}(x)=p_{\min}(x)$, the row has zero within-instance range and is
mapped to all-NaN before the common valid-instance filter. After
dropping degenerate rows, B4 retains
$556$ of $601$ instances and B5 retains $3565$ of the $4021$ raw
PROTEUS-2014 instances. Per-instance min-max preserves rankings and is
the standard convention adopted by ASlib analyses; under this
convention any instance with at least one strictly-better-than-worst
algorithm contributes a non-degenerate utility range.

For B5, where \texttt{runstatus}~$\neq$~\texttt{ok} (timeout, memout,
crash) or where \texttt{ok} runtime exceeds the cutoff, we apply the
standard PAR10 penalty $p(x, a) := 10 \cdot p_{\mathrm{cutoff}}$ before
the per-instance transform. The cutoff
$p_{\mathrm{cutoff}} = 3600$\,s is taken from the scenario's
\texttt{description.txt} (\texttt{algorithm\_cutoff\_time: 3600}); the
$1800$\,s cutoff sensitivity is reported in
Appendix~\ref{app:proteus-cutoff} and is not used for headline claims.

\section{Algorithm-family mappings}
\label{app:family-mapping}

\paragraph{B1 TabZilla \citep{mcelfresh2023tabzilla}.}
\textsc{TreeEnsemble}: CatBoost, LightGBM, XGBoost, RandomForest.
\textsc{Classic}: LinearModel, SVM, KNN, DecisionTree.
\textsc{Pretrained}: TabPFN.
\textsc{Deep}: MLP, rtdl\_MLP, rtdl\_ResNet, rtdl\_FTTransformer, SAINT,
TabTransformer, NAM, NODE, STG, TabNet, VIME, DANet, DeepFM.

\paragraph{B2 TabRepo \citep{salinas2024tabrepo}.}
\textsc{TreeEnsemble}: CAT, GBM, XGB, RF, XT.
\textsc{Classic}: KNN, LR.
\textsc{Deep}: FASTAI, FT\_TRANSFORMER, NN\_TORCH.
\textsc{Pretrained}: TABPFN.

\paragraph{B3 TALENT \citep{ye2024talent}.}
\textsc{TreeEnsemble}: catboost, lightgbm, RandomForest, xgboost.
\textsc{Classic}: knn, LogReg, NCM, NaiveBayes, svm.
\textsc{Pretrained}: tabpfn, tabpfn\_v2, tabpfn\_v2\_5, tabicl\_v2.
\textsc{Deep}: autoint, danets, dcn2, excelformer, ftt, grownet,
mlp, mlp\_plr, modernNCA, node, ptarl, ptaul, realmlp, resnet, snn,
switchtab, tabcaps, tabnet, tabr, tabtransformer, tangos.
The dataset's constant-reference baseline column is family-skipped (assigned
to no family; unreachable through the decomposed pipeline).

\paragraph{B4 ASlib MAXSAT-PMS-2016 \citep{li2009maxsat}.}
\textsc{BnB} (branch-and-bound): WMaxSatz09, WMaxSatz+, ahms-1.70.
\textsc{CoreGuided} (SAT-based core extraction; OLL-style): Open-WBO15,
Open-WBO16, mscg2015a, mscg2015b, maxino16-c10, maxino16-dis,
Naps-1.02-ms, Optiriss6, QMaxSAT14, QMaxSAT16UC, WPM3-2015-co.
\textsc{IHS} (implicit hitting set; hybrid SAT$+$MIP): maxhs-b,
LMHS-2016. \textsc{SLS} (stochastic local search; incl. SLS$+$BnB
hybrids): CCEHC2akms, CCLS2akms, ahms-ls-1.70.

\paragraph{B5 ASlib PROTEUS-2014 \citep{hurley2014proteus}.}
The PROTEUS portfolio decides between solving a CSP instance natively
or transforming it via one of three CSP-to-SAT encodings (direct,
support, direct-order) and dispatching to a SAT solver. The four
semantic families on the $22$ algorithms are:
\textsc{CSP-native}: abscon, choco, gecode, mistral\_nj.
\textsc{SAT-direct}: claspcnf\_direct, cryptominisat\_direct,
glucose\_direct, lingeling\_direct, minisat22\_direct, riss3g\_direct.
\textsc{SAT-support}: claspcnf\_support, cryptominisat\_support,
glucose\_support, lingeling\_support, minisat22\_support,
riss3g\_support.
\textsc{SAT-direct-order}: claspcnf\_directorder,
cryptominisat\_directorder, glucose\_directorder,
lingeling\_directorder, minisat22\_directorder, riss3g\_directorder.

The mapping is deterministic from the algorithm string suffix and
matches the original Proteus paper's hierarchy: CSP-native solvers
appear as one family, and the SAT-side $3 \times 6$ grid of
encodings $\times$ SAT solvers becomes three families of six.

\section{Synthetic toy example}
\label{app:toy}

The toy used in \S\ref{sec:algebra} consists of $n = 100$ synthetic
instances split evenly into two halves indexed by $h \in \{0,1\}$.
Algorithm utilities are
\begin{align*}
u(x \mid h\!=\!0) &= [1.0,\; 0.0,\; 0.5,\; 0.4] + \varepsilon, \\
u(x \mid h\!=\!1) &= [0.0,\; 1.0,\; 0.4,\; 0.5] + \varepsilon,
\end{align*}
with $\varepsilon \sim \mathcal{N}(0, 0.01^2)^4$ i.i.d.\ across instances
and algorithms. Pipeline $R_1$ uses the partition
$\Pi_1 = \{\{a_0,a_1\},\{a_2,a_3\}\}$ with an oracle family selector
($\hat{k}_{R_1}(x)$ always equals the family containing the optimal
algorithm) but a within-family selector that flips with probability
$\tau_1 = 0.5$ between the two members of the chosen family. Pipeline
$R_2$ uses $\Pi_2 = \{\{a_0,a_2\},\{a_1,a_3\}\}$ with family-selector
error rate $0.4$ and a nearly clean within-family selector
(flip rate $\tau_2 = 0.05$).

The toy construction yields
$\Spart(R_1) \approx 0.996$,
$\Spart(R_2) \approx 0.792$,
$\Setoe(R_1) \approx 0.520$,
$\Setoe(R_2) \approx 0.765$, hence
$\Apart \approx +0.20$, $\Aetoe \approx -0.24$, and $\rho \approx 2.20$.
The identity residual
$|\Apart - \Aetoe - (G(R_1) - G(R_2))|$ is below $10^{-15}$ for any
seed, illustrating Lemma~\ref{lem:absorption} numerically. This
construction is also the toy invoked at the end of
\S\ref{sec:margin-regret}: $R_1$ has the larger oracle margin but a
strictly larger within-family regret on the conflicted instances. This
is the mechanism behind the failure, but the formal cross-pipeline
sign change is the deployment-safety failure in
Corollary~\ref{cor:flat-safety}: $\Apart < G(R_1)-G(R_2)$, hence
$\Aetoe < 0$.

\section{Reversal taxonomy}
\label{app:taxonomy}

The full qualitative taxonomy of how a partition-level conclusion
transfers end-to-end, derived from~\eqref{eq:absorption}:
\begin{center}
\small
\begin{tabular}{lll}
\toprule
Regime & Algebraic condition & Practical reading \\
\midrule
Full transfer        & $G(R_1) - G(R_2) \approx 0$        & partition conclusion preserved \\
Amplification        & $\rho < 0$                         & advantage grows and sign holds \\
Partial absorption   & $0 < \rho < 1$                     & advantage shrinks but sign holds \\
Extreme absorption   & $\rho \approx 1$                   & advantage almost erased \\
Reversal             & $\rho > 1$, $\Apart > 0$           & partition winner becomes e2e loser \\
Opposite reversal    & $\Apart < 0$, $\Aetoe > 0$         & partition loser becomes e2e winner \\
\bottomrule
\end{tabular}
\end{center}
Reversal and opposite reversal are the two sign-reversal regimes; our
empirical study observes standard-reversal point estimates on two
benchmarks but does not observe opposite reversals once
$\Spart$/$\Setoe$/$G$ are reported on a common instance set.
Amplification is mathematically possible when the partition-level
winner has a smaller deployment-fidelity gap than its comparator; it
is not needed for the headline empirical claims.

\section{Reporting checklist and decision framework}
\label{app:guidance}

\paragraph{Decision framework.} Table~\ref{tab:decisions} maps research
questions to recommended metrics. Partition score is appropriate when
the question is about the family partition itself; it is insufficient
when the question is about a deployable system or about comparing two
decomposed methods.

\begin{table}[t]
\small
\centering
\caption{Metric-reporting decision framework grounded in our findings.}
\label{tab:decisions}
\begin{tabularx}{\linewidth}{@{}>{\raggedright\arraybackslash}p{0.36\linewidth}>{\raggedright\arraybackslash}p{0.18\linewidth}>{\raggedright\arraybackslash}X@{}}
\toprule
Research question & Metric & Why \\
\midrule
Is the family partition useful in principle? &
  $\Spart$ + VBS upper bound &
  measures whether the selected family contains strong algorithms \\
Is the system deployable? &
  $\Setoe$ &
  deployment uses a learned within-family selector, not an oracle \\
Does method $R_1$ beat $R_2$? &
  both $\Apart$ and $\Aetoe$ &
  B2/B4 show sign-changing point estimates; B1/B5 show strong absorption \\
Is the within-family selector adequate? &
  $G(R)$ with CI &
  Lemma~\ref{lem:gap-identity}: $G$ equals within-family regret \\
Is partition advantage reliable? &
  $\rho$ &
  quantifies absorption / reversal \\
\bottomrule
\end{tabularx}
\end{table}

\paragraph{Full reporting checklist.} A decomposed-AS report should
include: (i) $\Spart$ and $\Setoe$ side by side; (ii) $G(R)$ with
bootstrap CI; (iii) for comparisons, $\Apart, \Aetoe, \rho$, with
$\rho$ marked unstable when $|\Apart| < 5\times 10^{-3}$;
(iv) the within-family selector class and any oracle used;
(v) the family mapping \emph{and the partition's semantic basis} ---
paradigm, encoding, lineage, random, or cluster --- so that readers can
judge whether the partition matches the deployment hierarchy;
(vi) the \emph{feature availability policy} --- whether the family
selector uses features that are only computable after a representation
or encoding choice has been made;
(vii) for runtime scenarios, the timeout, PAR10 convention, and
missing-value handling;
(viii) the common valid-instance set used per comparison and the
empirical residuals of Lemmas~\ref{lem:gap-identity}--\ref{lem:absorption}
as a consistency check;
(ix) for every method comparison, the \emph{partition-only
identification interval} $[\Apart - W_{R_1}, \Apart + W_{R_2}]$ from
Theorem~\ref{thm:identification} and whether it strictly crosses zero ---
the partition-level report does not certify a deployable winner unless
the interval is strictly signed.

\section{Proofs}
\label{app:proofs}

\paragraph{Lemma~\ref{lem:gap-identity} (gap identity).}
Substituting~\eqref{eq:spart} and~\eqref{eq:se2e} into the definition of
$G(R)$ and using linearity of expectation gives the first equality. For
each $x$, $a_{\famhat(x)}^{*}(x)$ is by definition the maximiser of
$u(x,\cdot)$ over $F_{\famhat(x)}$, and $g_{R,\famhat(x)}(x) \in
F_{\famhat(x)}$, so the pointwise difference is nonnegative. Equality
in expectation forces equality almost surely. \qed

\paragraph{Lemma~\ref{lem:absorption} (advantage absorption).}
By definition $\Setoe(R) = \Spart(R) - G(R)$. Substituting into
$\Aetoe = \Setoe(R_1) - \Setoe(R_2)$ and rearranging yields
\eqref{eq:absorption}. \qed

\paragraph{Corollary~\ref{cor:flat-safety} (general deployment-safety).}
By Lemma~\ref{lem:absorption},
$\Aetoe = \Apart - [G(R) - G(F)] = \Apart - \Delta G$, so
$\Aetoe > 0 \iff \Apart > \Delta G$. Setting $G(F) = 0$ recovers the
canonical form. \qed

\paragraph{Theorem~\ref{thm:margin-regret} (margin--regret stability).}
$v_q(x) = m_q(x) - r_q(x)$ for every family $q$, so
$v_q(x) - v_s(x) = (m_q - m_s) - (r_q - r_s)$. Maximising over $q$
gives the regret expression; non-negativity of regret yields the
equivalence in~\eqref{eq:stability-condition}. \qed

\paragraph{Theorem~\ref{thm:identification} (partition-only identification).}
By Lemma~\ref{lem:absorption} applied to any $(g_1, g_2)$,
$\Aetoe(g_1, g_2) = \Apart - (G_{g_1}(R_1) - G_{g_2}(R_2))$.
Independently varying $g_1, g_2$ over $\mathcal{G}_{R_i}$ bounds
$G_{g_i}(R_i)$ by $[\underline{G}(R_i), \overline{G}(R_i)]$, so
$\Aetoe$ is bounded by the interval in~\eqref{eq:identification-interval};
attainability of each envelope yields sharpness, otherwise the bound is
approached but not attained. Under the unrestricted family-respecting
class, $\underline{G}(R) = 0$ is attained by the within-family oracle
selector $g(x) \in \arg\max_{a \in F_{\famhat(x)}} u(x, a)$ and
$\overline{G}(R) = W_R$ is attained by the pointwise within-family
worst-case selector $g(x) \in \arg\min_{a \in F_{\famhat(x)}} u(x, a)$;
both are measurable on a finite algorithm pool. \qed

\section{Common-instance evaluation and identity residuals}
\label{app:nan-handling}\label{app:identity-residuals}

The identities in Lemmas~\ref{lem:gap-identity}--\ref{lem:absorption}
hold pointwise. For each (benchmark, pair) we therefore compute
$\Spart(R_1)$, $\Setoe(R_1)$, $\Spart(R_2)$, $\Setoe(R_2)$ on the
\emph{four-way intersection} of finite values, i.e.\ the test instances
on which both pipelines have a finite partition oracle and a finite
end-to-end utility; the resulting empirical residuals
$|\Apart - \Aetoe - (G(R_1) - G(R_2))|$ are at machine precision
($\le 5 \times 10^{-16}$) on all five benchmarks.

Figure~\ref{fig:identity} visualises the two identities as a sanity
check. The left panel plots within-family regret on the predicted
family against $G(R)$: by Lemma~\ref{lem:gap-identity} the two are
equal in expectation, so all benchmarks lie on $y=x$. The right panel
plots $\Apart - \Aetoe$ versus $G(R_1) - G(R_2)$ for the
pairwise-vs-multiclass comparison on every benchmark; by
Lemma~\ref{lem:absorption} all points lie on the diagonal (residuals
$\le 5\times 10^{-16}$). These plots are not statistical fits but
empirical realisations of the lemmas on per-benchmark estimates.

\begin{figure}[t]
\centering
\begin{minipage}{0.49\linewidth}
\includegraphics[width=\linewidth]{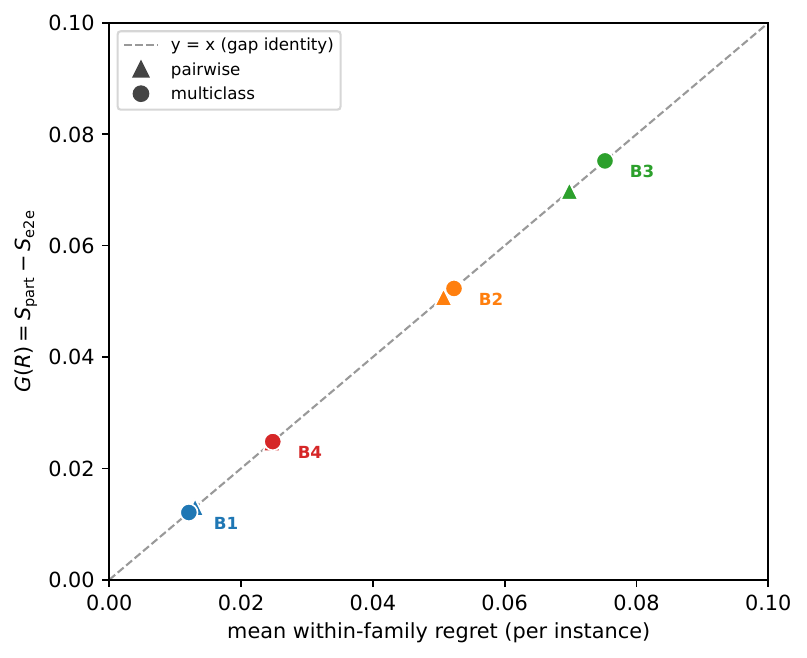}
\end{minipage}\hfill
\begin{minipage}{0.49\linewidth}
\includegraphics[width=\linewidth]{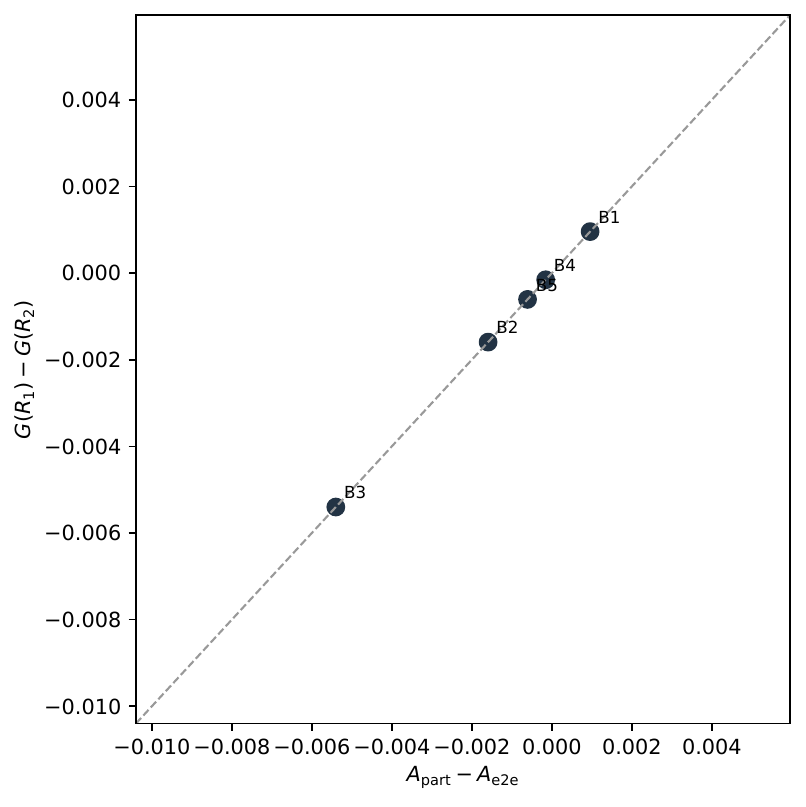}
\end{minipage}
\caption{\emph{Left:} per-benchmark within-family regret on the
predicted family against $G(R)$ (Lemma~\ref{lem:gap-identity}).
\emph{Right:} $\Apart - \Aetoe$ versus $G(R_1) - G(R_2)$ for the
pairwise-vs-multiclass comparison on every benchmark
(Lemma~\ref{lem:absorption}; residuals $\le 5\times 10^{-16}$).}
\label{fig:identity}
\end{figure}

This common-instance protocol is essential: without it, marginal
means would average over different row sets whenever a pipeline's
partition oracle is undefined on some test instances while the
deployable selector falls back to a penalty value, and apparent
apparent sign changes can be induced by aggregation alone. For TALENT (B3),
using independent valid masks can create marginal artefacts. We
therefore compute each comparison on a common finite utility mask and
treat common-instance evaluation as part of the reporting checklist
(\S\ref{sec:guidance}).

Flat baselines are singleton-family controls. For these controls the
evaluator enforces the row-wise invariant
$u_{\mathrm{part}}(x)=u_{\mathrm{e2e}}(x)$ after missing-value fallback:
the same selected algorithm and the same per-instance penalty are used
for both scores. Consequently all flat rows have $G(F)=0$, including
TALENT (B3); any nonzero singleton-flat gap would indicate asymmetric
NaN handling rather than a within-family oracle effect.

\section{Decomposed pipeline pseudocode}
\label{app:pseudocode}

The protocol used to compute $\Spart$, $\Setoe$, and $G(R)$ in
\S\ref{sec:empirical} is summarised below.
Stratification keys are the per-instance argmax over $\mathcal{A}$,
with rare classes collapsed to a synthetic \texttt{\_\_other\_\_}
class to avoid singleton folds. The within-family selector is fit
per family using the family-restricted columns of $M$ on the
training rows.

\begin{quote}\small
\textbf{Algorithm: Decomposed pipeline evaluator (one seed pass; the outer driver iterates this over $5$ seeds and averages).}\\
\textbf{Input.} Utility matrix $M \in \mathbb{R}^{n \times A}$,
feature matrix $Z \in \mathbb{R}^{n \times d}$, family partition
$\Pi = \{F_1,\dots,F_K\}$, $K_{\mathrm{cv}}$-fold splits
$\{(\mathrm{tr}_t, \mathrm{te}_t)\}_{t=1}^{K_{\mathrm{cv}}}$.\\
\textbf{Output.} Per-instance utility vectors
$u_{\mathrm{part}}, u_{\mathrm{e2e}} \in \mathbb{R}^n$.

\begin{enumerate}
  \itemsep0em
  \item Initialise $u_{\mathrm{part}} \leftarrow \mathbf{0}_n$, $u_{\mathrm{e2e}} \leftarrow \mathbf{0}_n$.
  \item For each fold $t = 1,\dots,K_{\mathrm{cv}}$:
    \begin{enumerate}
      \itemsep0em
      \item Fit family selector $h_R$ on $(Z[\mathrm{tr}_t],\, M[\mathrm{tr}_t],\, \Pi)$.
      \item For each family $k$: fit within-family selector $g_{R,k}$ on $(Z[\mathrm{tr}_t],\, M[\mathrm{tr}_t, F_k])$.
      \item For each $i \in \mathrm{te}_t$: set $\hat{k} \leftarrow h_R(Z[i])$, $\hat{a} \leftarrow g_{R, \hat{k}}(Z[i])$, $u_{\mathrm{part}}[i] \leftarrow \max_{a \in F_{\hat{k}}} M[i, a]$ (penalty fallback if no member is finite), and $u_{\mathrm{e2e}}[i] \leftarrow M[i, \hat{a}]$ (the same per-instance penalty if missing). For singleton flat baselines, tie the two assignments after fallback so $u_{\mathrm{part}}[i]=u_{\mathrm{e2e}}[i]$.
    \end{enumerate}
  \item \textbf{Return} $(u_{\mathrm{part}}, u_{\mathrm{e2e}})$.
\end{enumerate}
\end{quote}

The evaluation runs the algorithm above for each of $5$
seeds, averages per-instance utilities across seeds, and reports
$\Spart = \mathrm{mean}(u_{\mathrm{part}})$,
$\Setoe = \mathrm{mean}(u_{\mathrm{e2e}})$, and
$G = \mathrm{mean}(u_{\mathrm{part}} - u_{\mathrm{e2e}})$ along with
$95\%$ percentile bootstrap CIs over per-instance differences.

\section{Gap-corrected reporting}
\label{app:correction}

We evaluate the gap-correction proposed in \S\ref{sec:guidance} via
training-side validation. For each outer fold, the outer test fold is
kept untouched. We split the outer training fold into validation folds,
fit the two candidate pipelines on the inner training portion, estimate
$\widehat{G}^{\mathrm{val}}(R_1)-\widehat{G}^{\mathrm{val}}(R_2)$ on
the held-out validation portion, and then correct
$\widehat{A}_{\mathrm{part}}^{\mathrm{test}}$ on the original outer
test fold. The correction therefore uses only training-side data for
the gap estimate.

Figure~\ref{fig:correction} reports the result. On the four current
decomp-vs-flat sign-changing cells (B2/B4), correction recovers the
point-estimate end-to-end sign and reduces MAE by $25$--$65\%$; the
fold-level corrected sign accuracy is $0.68$--$0.80$ on B2 and
$0.56$--$0.60$ on B4. A split-sensitivity check on these four cells
keeps the point-estimate sign correct for validation-fold counts
$3$, $5$, and $10$, and for two additional $5$-fold validation split
seeds; MAE reduction ranges from $37$--$65\%$ on B2 and $9$--$31\%$
on B4. On non-sign-changing but strongly absorbed cells, it
reduces MAE by $60$--$65\%$ on B3 and $87$--$88\%$ on B5, while B1 is
the cautionary case: the true $\Aetoe$ is only $0.001$--$0.002$, and
correction over-subtracts on the point estimate. We therefore use
correction as a diagnostic and still require direct $\Setoe$ reporting;
we do not test robustness under explicit covariate or task-distribution
shift.

\begin{figure}[H]
\centering
\includegraphics[width=\linewidth]{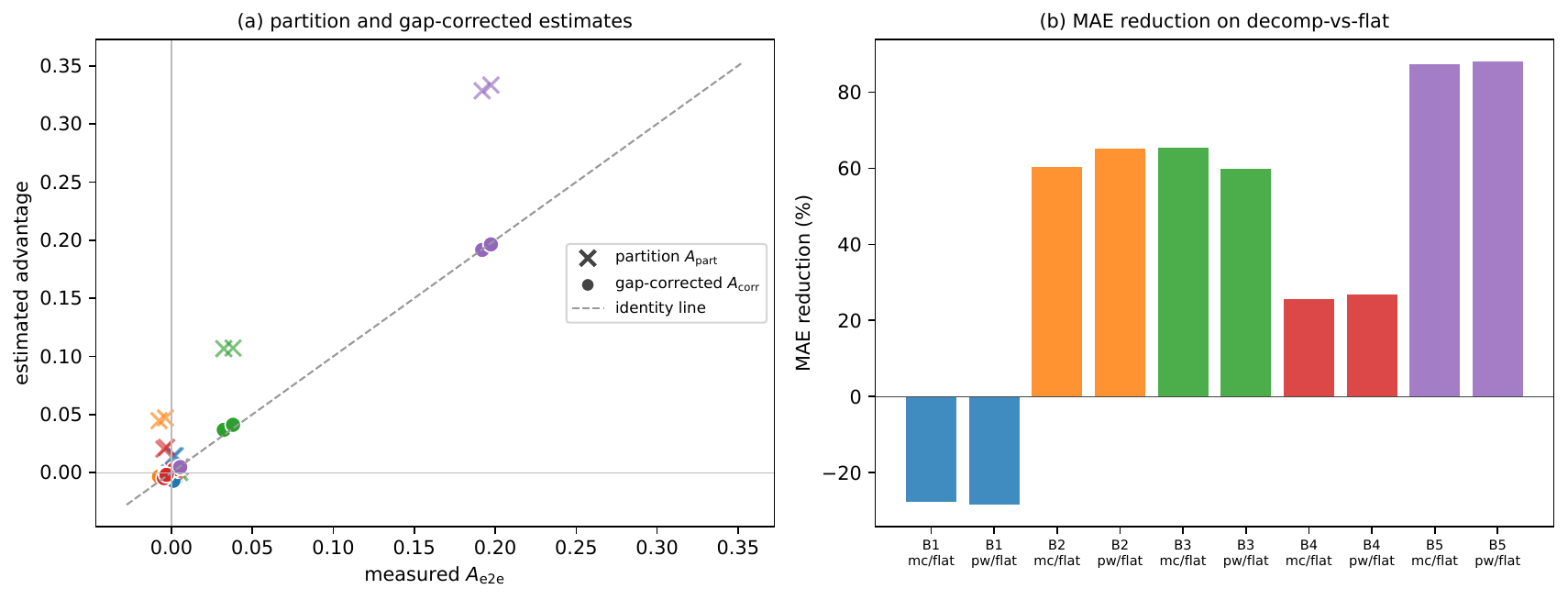}
\caption{Gap-corrected vs.\ raw partition advantage on all $15$
pairwise comparisons. \emph{(a)} crosses are raw $A_{\mathrm{part}}$,
filled circles are corrected $\widehat{A}_{\mathrm{corr}}$; corrected
points are closer to the $y=x$ identity than the raw partition estimates.
\emph{(b)} MAE reduction (\%) for each decomp-vs-flat comparison;
gain is positive on B2--B5 and negative on the B1 high-absorption cells,
where the deployable advantage is close to zero and validation-side
gap estimates can over-correct.}
\label{fig:correction}
\end{figure}

\section{Family-mapping sensitivity}
\label{app:family-sensitivity}

For B4 we re-evaluate the deployment-fidelity gap and the
decomp-vs-flat advantage transfer under three alternative family
partitions in addition to the default semantic taxonomy: (P2) a
coarse two-family collapse, (P3) a performance-cluster partition with
$K=4$ obtained by $k$-means on per-algorithm utility profiles, and
(P4) a random balanced $K=4$ placebo averaged over five seeds.
Figure~\ref{fig:family-sens} plots $G(R)$, $\Apart$, and $\Aetoe$
side-by-side per partition for B4; Table~\ref{tab:family-sens} reports
the numbers and the safety-condition flag. For B5 (PROTEUS-2014) we
report only the primary semantic partition $\Pi_{\text{P1}}$ here;
coarser K-sensitivity for B5 is in the companion reproducibility package. The
primary result is that the positive deployment-fidelity gap and the
failure of the flat-deployment safety condition on B4 are robust to
the choice of family taxonomy: every partition we test on B4 produces
$G(R)$ exceeding $\Apart$.

\begin{figure}[H]
\centering
\includegraphics[width=0.65\linewidth]{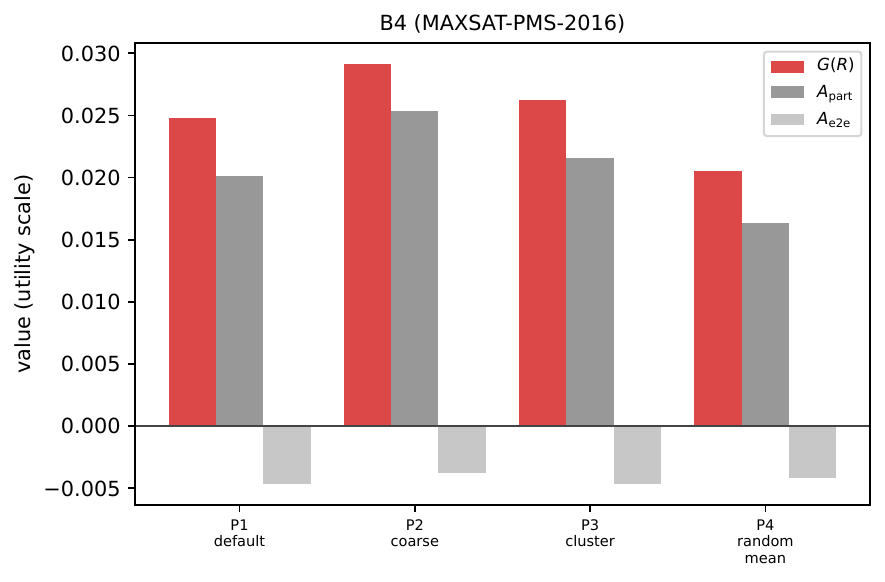}
\caption{Family-mapping sensitivity on B4. Bars are $G(R)$, $\Apart$,
and $\Aetoe$ per partition kind. B4 fails the decomp-vs-flat safety
condition $\Apart > G(R)$ for every partition tested, showing that the
gap magnitude is not driven solely by the default taxonomy. B5
mapping sensitivity beyond the primary semantic PROTEUS partition is
not used for paper claims.}
\label{fig:family-sens}
\end{figure}

\setcounter{table}{3}
\begin{table}[H]
\small
\centering
\caption{Family-mapping sensitivity. Each row is one (benchmark,
partition) cell with the resulting gap and decomp-vs-flat advantages.
``safe'' = general safety condition $\Apart > \Delta G$ where
$\Delta G = G(R) - G(F)$. P4 (random) is averaged over five seeds.
B4 fails the safety condition under every partition we test (default,
coarse, performance-cluster, random placebo); B5 (PROTEUS-2014) reports
the primary semantic partition $\Pi_{\text{P1}}$ only --- coarser
K-sensitivity for B5 is in the companion reproducibility package. Family mapping
changes the gap magnitude but does not remove it.}
\label{tab:family-sens}
\begin{tabular}{llrrrrl}
\toprule
ID & Partition & $K$ & $\Apart$ & $\Aetoe$ & $G(R)$ & safe? \\
\midrule
B4 & P1 default              & $4$ & $+0.0202$ & $-0.0047$ & $0.0248$ & no \\
B4 & P2 coarse               & $2$ & $+0.0254$ & $-0.0038$ & $0.0291$ & no \\
B4 & P3 performance-cluster  & $4$ & $+0.0216$ & $-0.0047$ & $0.0262$ & no \\
B4 & P4 random placebo (mean of 5) & $4$ & $+0.0163$ & $-0.0042$ & $0.0205$ & no \\
\midrule
B5 & P1 semantic (CSP-native + 3 SAT enc.) & $4$ & $+0.3333$ & $+0.1973$ & $0.1360$ & yes \\
\bottomrule
\end{tabular}
\end{table}

\section{Deployment-aware family selector}
\label{app:deployment-aware}

We compare the default oracle-trained family selector
($\arg\max_k \max_{a \in F_k} u(x, a)$ on the training set) against a
deployment-aware family selector that targets the cross-fitted
deployable family utility
$v_k(x) = u(x, \hat{g}_k^{(-m)}(x))$, where $\hat{g}_k^{(-m)}$ is the
within-family selector trained on the inner training fold not
containing $x$. Both pipelines share the
\texttt{PerAlgRegressionWithin} within-family selector.
Per-benchmark headline numbers appear in
Table~\ref{tab:deployment-aware}.

\begin{table}[H]
\small
\centering
\caption{Deployment-aware family selector (DA) vs.\ oracle-trained
family selector (O), with both sharing
\texttt{PerAlgRegressionWithin}. $\Delta_X = X_{\mathrm{DA}} - X_{\mathrm{O}}$.
DA shrinks $G(R)$ on every benchmark; it improves $S_{\mathrm{e2e}}$
on B1--B4 and trades off slightly on B5 PROTEUS-2014, where the
within-family utility range is large enough that targeting deployable
utility moves family decisions toward families with smaller oracle
margin. Sign-changing point estimates relative to flat persist on B2 and B4 --- the
safety condition remains the primary diagnostic.}
\label{tab:deployment-aware}
\begin{tabular}{lrrrrrrr}
\toprule
ID & $S_{\mathrm{e2e}}^{\mathrm{O}}$ & $S_{\mathrm{e2e}}^{\mathrm{DA}}$ & $\Delta S_{\mathrm{e2e}}$ & $G^{\mathrm{O}}$ & $G^{\mathrm{DA}}$ & $\Delta G$ & Sign chg. O / DA \\
\midrule
B1 & $0.9165$ & $0.9174$ & $+0.0009$ & $0.0121$ & $0.0114$ & $-0.0006$ & no / no \\
B2 & $0.8961$ & $0.9007$ & $+0.0046$ & $0.0523$ & $0.0504$ & $-0.0019$ & yes / yes \\
B3 & $0.7605$ & $0.7658$ & $+0.0053$ & $0.0752$ & $0.0723$ & $-0.0029$ & no  / no  \\
B4 & $0.9700$ & $0.9710$ & $+0.0011$ & $0.0248$ & $0.0240$ & $-0.0008$ & yes / yes \\
B5 & $0.8129$ & $0.8085$ & $-0.0045$ & $0.1365$ & $0.1286$ & $-0.0079$ & no  / no  \\
\bottomrule
\end{tabular}
\end{table}

\section{PROTEUS-2014 feature-policy and cutoff sensitivity}
\label{app:proteus-features}

The PROTEUS-2014 ASlib scenario provides four feature steps with $36$,
$54$, $54$, and $54$ features respectively for a total of $198$
features. The \texttt{csp} step contains features computable on the
original CSP instance; \texttt{direct}, \texttt{support}, and
\texttt{directorder} contain features that are only well-defined after
the corresponding CSP-to-SAT encoding has been applied. Our primary
analysis uses the \texttt{csp} step alone so that the family selector
makes its CSP-or-SAT-and-encoding decision on information available
before any encoding is fixed.

For the all-feature sensitivity, we re-run the B5 main pipeline with
all $198$ features supplied to the family selector and the same
within-family selector class. Headline numbers under both feature
policies appear in Table~\ref{tab:proteus-feature-policy}. The
all-feature setting is a stronger-information setting that gives the
family selector access to encoding-conditional features it would not
have at deployment time; we therefore interpret the \texttt{csp}-step
configuration as the primary deployment model and the all-feature
setting as a sensitivity check.

\label{app:proteus-cutoff}
The PROTEUS-2014 ASlib scenario reports algorithm runtimes censored at
the scenario cutoff of $3600$\,s and we use this cutoff for the
primary analysis. We additionally re-run B5 with the cutoff set to
$1800$\,s: any \texttt{ok} run whose runtime exceeds $1800$\,s is
re-classified as a timeout under the new (shorter) cutoff and
penalised at $\mathrm{PAR10} = 18000$\,s. We do not extrapolate beyond
the scenario cutoff (e.g.\ to $7200$\,s) because uncensored raw run
logs are not available. Table~\ref{tab:proteus-cutoff} reports
$\Spart, \Setoe, G(R), \Apart, \Aetoe, \rho$ under both cutoffs; the
direction of every conclusion is preserved.

\begin{table}[H]
\small
\centering
\caption{PROTEUS-2014 feature-policy sensitivity. Headline pipeline
quantities under (a) the primary \texttt{csp} feature step ($36$
features) and (b) the all-feature ($198$) sensitivity. The
\texttt{csp} setting is the deployment model (encoding-conditional
features are not visible to the family selector); the all-feature
setting is a stronger-information setting that shows what would happen
if the family selector could see encoding features computable only
after an encoding has been chosen.}
\label{tab:proteus-feature-policy}
\begin{tabular}{llrrr}
\toprule
Feature policy & Pipeline & $\Spart$ & $\Setoe$ & $G$ \\
\midrule
csp ($36$) primary  & pairwise   & $0.954$ & $0.820$ & $0.134$ \\
csp ($36$) primary  & multiclass & $0.951$ & $0.816$ & $0.135$ \\
csp ($36$) primary  & flat       & $0.622$ & $0.622$ & $0.000$ \\
\midrule
all ($198$) sensit. & pairwise   & $0.945$ & $0.845$ & $0.100$ \\
all ($198$) sensit. & multiclass & $0.938$ & $0.840$ & $0.098$ \\
all ($198$) sensit. & flat       & $0.669$ & $0.669$ & $0.000$ \\
\bottomrule
\end{tabular}\\
{\footnotesize This sensitivity is exploratory and is not used for
headline claims. Directional conclusions are unchanged under both
feature policies: $G > 0$ on both decomposed pipelines, and the
all-feature setting reduces $G$ by roughly $25\%$ ($0.135 \to 0.098$
on multiclass, $0.134 \to 0.100$ on pairwise) via improvement in
$\Setoe$ rather than reduction in $\Spart$ --- the within-family
selector benefits more from extra features than the family oracle
does. Full outputs are included in the companion reproducibility
package.}
\end{table}

\begin{table}[H]
\small
\centering
\caption{PROTEUS-2014 cutoff sensitivity. Headline pipeline quantities
under the $3600$\,s scenario cutoff (primary) and a $1800$\,s
sensitivity in which any \texttt{ok} run exceeding $1800$\,s is
re-classified as a timeout and penalised at PAR10. We do not extrapolate
beyond the scenario cutoff because uncensored raw run logs are not
available.}
\label{tab:proteus-cutoff}
\begin{tabular}{llrrr}
\toprule
Cutoff & Pipeline & $\Spart$ & $\Setoe$ & $G$ \\
\midrule
$3600$\,s (primary) & pairwise   & $0.954$ & $0.820$ & $0.134$ \\
$3600$\,s (primary) & multiclass & $0.951$ & $0.816$ & $0.135$ \\
$3600$\,s (primary) & flat       & $0.622$ & $0.622$ & $0.000$ \\
\midrule
$1800$\,s sensitivity & pairwise   & $0.953$ & $0.819$ & $0.134$ \\
$1800$\,s sensitivity & multiclass & $0.948$ & $0.812$ & $0.136$ \\
$1800$\,s sensitivity & flat       & $0.616$ & $0.616$ & $0.000$ \\
\bottomrule
\end{tabular}\\
{\footnotesize This sensitivity is exploratory and is not used for
headline claims. Directional conclusions are unchanged: $G(R)$ is
essentially invariant under the cutoff shortening (pairwise gap
$0.134 \to 0.134$, multiclass $0.135 \to 0.136$); halving the cutoff
drops $14$ all-timeout instances ($n: 3565 \to 3551$) and pulls the
flat baseline's PAR10 slightly downward without changing the
conclusion. Full outputs are included in the companion reproducibility
package.}
\end{table}

\section{Metric-transform sensitivity on B4--B5}
\label{app:native-metric}

This appendix checks whether the B4 sign-change and B5 absorption
conclusions are artefacts of the per-instance min--max utility
transform of \S\ref{sec:protocol}. The substantive sensitivity check is
decomp-vs-flat advantage transfer under a second common $[0,1]$
transform, per-instance rank utility. Appendix~\ref{app:transform}
defines the primary utility from raw PAR10/runtime before
row-normalisation, but the downstream prediction-cache diagnostic used
for this sensitivity stores normalised utilities and predictions rather
than a reusable raw-cost matrix aligned to every recomputed comparison
mask. The ``native'' column in the released diagnostic therefore falls
back to a monotone $-\,$utility proxy and is reported only as an
implementation sanity check, not as independent raw-second evidence.
We consequently avoid using the pointwise non-negativity of a
partition-oracle cost gap as robustness evidence.

\begin{table}[H]
\small
\centering
\caption{Metric-transform sensitivity for decomp-vs-flat advantage
transfer on B4--B5. Positive advantages favour decomposition. Min--max
is the headline utility; rank utility maps the best algorithm on an
instance to $1$ and the worst to $0$. Rank utility confirms the B4
sign change and the strong B5 absorption pattern, although B5's
multiclass rank end-to-end margin is effectively zero.}
\label{tab:native-metric}
\begin{tabular}{llrrrr}
\toprule
ID & Pair & $\Apart^\mathrm{mm}$ & $\Aetoe^\mathrm{mm}$ & $\Apart^\mathrm{rank}$ & $\Aetoe^\mathrm{rank}$ \\
\midrule
B4 & mc vs flat & $+0.020$ & $-0.0047$ & $+0.086$ & $-0.162$ \\
B4 & pw vs flat & $+0.022$ & $-0.0031$ & $+0.087$ & $-0.160$ \\
\midrule
B5 & mc vs flat & $+0.328$ & $+0.192$ & $+0.237$ & $-0.0004$ \\
B5 & pw vs flat & $+0.333$ & $+0.197$ & $+0.243$ & $+0.0047$ \\
\bottomrule
\end{tabular}\\
{\footnotesize Full long-form outputs, including the native proxy
sanity check, are in \texttt{results/exp10\_native\_metric}. Because
the proxy is monotone-equivalent to the headline utility, it is not
used as separate robustness evidence.}
\end{table}

\section{Within-family selector robustness sweep}
\label{app:exp5-b4b5}

\begin{figure}[t]
\centering
\includegraphics[width=0.55\linewidth]{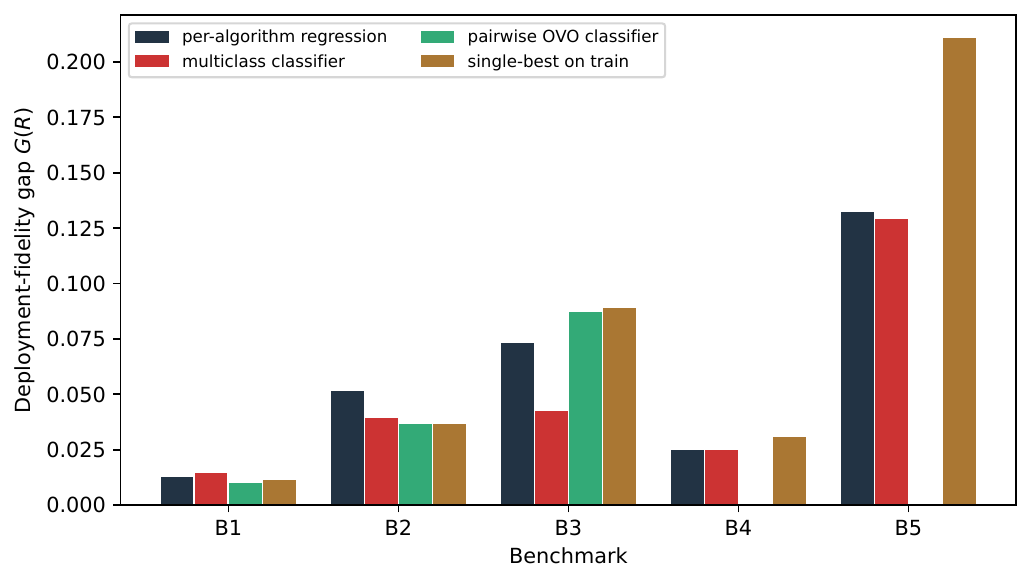}
\caption{$G(R)$ averaged over family-selector choice, by within-family
selector class. B1--B3 sweep four classes (per-algorithm GBDT,
multiclass, pairwise OVO, single-best); B4--B5 sweep three (pairwise
OVO is omitted on ASlib due to its $\mathcal{O}(|F|^2)$ solver-pair
cost). Family-selector class changes $\Spart$ negligibly (within
$5\times 10^{-3}$), while within-family selector class shifts $G$ by
up to $2.2\times$ on TALENT (B3) and produces a $0.21$ gap on B5
under the main-sweep single-best-on-train selector.}
\label{fig:robust}
\end{figure}

\S\ref{sec:robustness} summarises the within-family selector
sensitivity sweep, visualised in Figure~\ref{fig:robust}. On B1--B3
we sweep four classes (per-algorithm GBDT, multiclass, pairwise OVO,
single-best-on-train); on B4--B5 we sweep three
(pairwise-within-family is omitted because its $O(|F|^2)$ cost on the
larger ASlib pools is prohibitive). $G(R)$ ranges
$[0.025, 0.031]$ on B4 (spread $1.28\times$) and $[0.128, 0.211]$ on
B5 (spread $1.65\times$). Even the best within-family selector in
this sweep produces a substantial gap on both benchmarks; the largest
main-sweep PROTEUS-2014 value is $G=0.211$ under the
single-best-on-train selector. The deployment-fidelity gap is
therefore not a property of any single weak within-family selector.

Table~\ref{tab:exp5-full} reports the full $24$-row robustness sweep
on B1--B3 (2 family selectors $\times$ 4 within-family selectors
$\times$ 3 benchmarks; the pairwise-within-family configuration is
omitted on B4--B5 due to its $O(|F|^2)$ cost). Family-selector class
changes $\Spart$ negligibly (the row pairs comparing PW-family vs
MC-family at fixed within-class agree to $\le 5 \times 10^{-3}$ on
every benchmark), while within-family selector class shifts $G$ by up
to $1.44\times$ on B1 and B2 and by $2.23\times$ on B3 (TALENT, $G
\in [0.041, 0.092]$).

\begin{table}[H]
\small
\centering
\caption{Within-family selector robustness sweep on
B1--B3. Abbreviations: family selectors PW = pairwise OVO, MC =
multiclass softmax; within-family selectors GBDT = per-algorithm
gradient-boosted regressor, MC = multiclass classifier, PW = pairwise
OVO classifier, SB = single-best-on-train. Each row is computed on its
row-specific common valid-instance mask; $G=\Spart-\Setoe$ holds before
rounding.}
\label{tab:exp5-full}
\begin{tabular}{llrrr}
\toprule
ID & Fam.\,/\,Within & $\Spart$ & $\Setoe$ & $G$ \\
\midrule
B1 & PW / GBDT & $0.9284$ & $0.9154$ & $0.0130$ \\
B1 & PW / MC   & $0.9284$ & $0.9141$ & $0.0143$ \\
B1 & PW / PW   & $0.9284$ & $0.9180$ & $0.0104$ \\
B1 & PW / SB   & $0.9284$ & $0.9169$ & $0.0115$ \\
B1 & MC / GBDT & $0.9286$ & $0.9165$ & $0.0121$ \\
B1 & MC / MC   & $0.9286$ & $0.9143$ & $0.0143$ \\
B1 & MC / PW   & $0.9286$ & $0.9187$ & $0.0099$ \\
B1 & MC / SB   & $0.9286$ & $0.9171$ & $0.0115$ \\
B2 & PW / GBDT & $0.9508$ & $0.9001$ & $0.0507$ \\
B2 & PW / MC   & $0.9508$ & $0.9119$ & $0.0389$ \\
B2 & PW / PW   & $0.9508$ & $0.9138$ & $0.0370$ \\
B2 & PW / SB   & $0.9508$ & $0.9145$ & $0.0363$ \\
B2 & MC / GBDT & $0.9484$ & $0.8961$ & $0.0523$ \\
B2 & MC / MC   & $0.9484$ & $0.9090$ & $0.0393$ \\
B2 & MC / PW   & $0.9484$ & $0.9120$ & $0.0364$ \\
B2 & MC / SB   & $0.9484$ & $0.9111$ & $0.0373$ \\
B3 & PW / GBDT & $0.8361$ & $0.7654$ & $0.0707$ \\
B3 & PW / MC   & $0.8361$ & $0.7951$ & $0.0410$ \\
B3 & PW / PW   & $0.8361$ & $0.7511$ & $0.0850$ \\
B3 & PW / SB   & $0.8361$ & $0.7501$ & $0.0860$ \\
B3 & MC / GBDT & $0.8357$ & $0.7605$ & $0.0752$ \\
B3 & MC / MC   & $0.8357$ & $0.7917$ & $0.0440$ \\
B3 & MC / PW   & $0.8357$ & $0.7460$ & $0.0897$ \\
B3 & MC / SB   & $0.8357$ & $0.7442$ & $0.0915$ \\
\bottomrule
\end{tabular}
\end{table}

\section{Full numerical tables for gap and advantage-transfer analyses}
\label{app:full-tables}

Table~\ref{tab:exp2-full} reports per-cell $(\Spart, \Setoe, G)$ with
$95\%$ CIs and paired Wilcoxon $p$-values for all $15$ (benchmark,
decomposition) cells (Figure~\ref{fig:forest}).
Table~\ref{tab:exp3-full} reports the full $15$-pair advantage transfer
table including the pairwise-vs-multiclass cells omitted from
Table~\ref{tab:transfer} for space.
Figure~\ref{fig:scatter} visualises the same $15$ comparisons in
$(\Apart, \Aetoe)$ space. Figure~\ref{fig:safety} plots the
deployment-safety condition for the $10$ decomp-vs-flat comparisons
referenced in \S\ref{sec:rq2-flat}, and Table~\ref{tab:interval-summary}
summarises the partition-only identification intervals
(Theorem~\ref{thm:identification}) referenced in \S\ref{sec:rq3}.

\begin{figure}[t]
\centering
\includegraphics[width=0.60\linewidth]{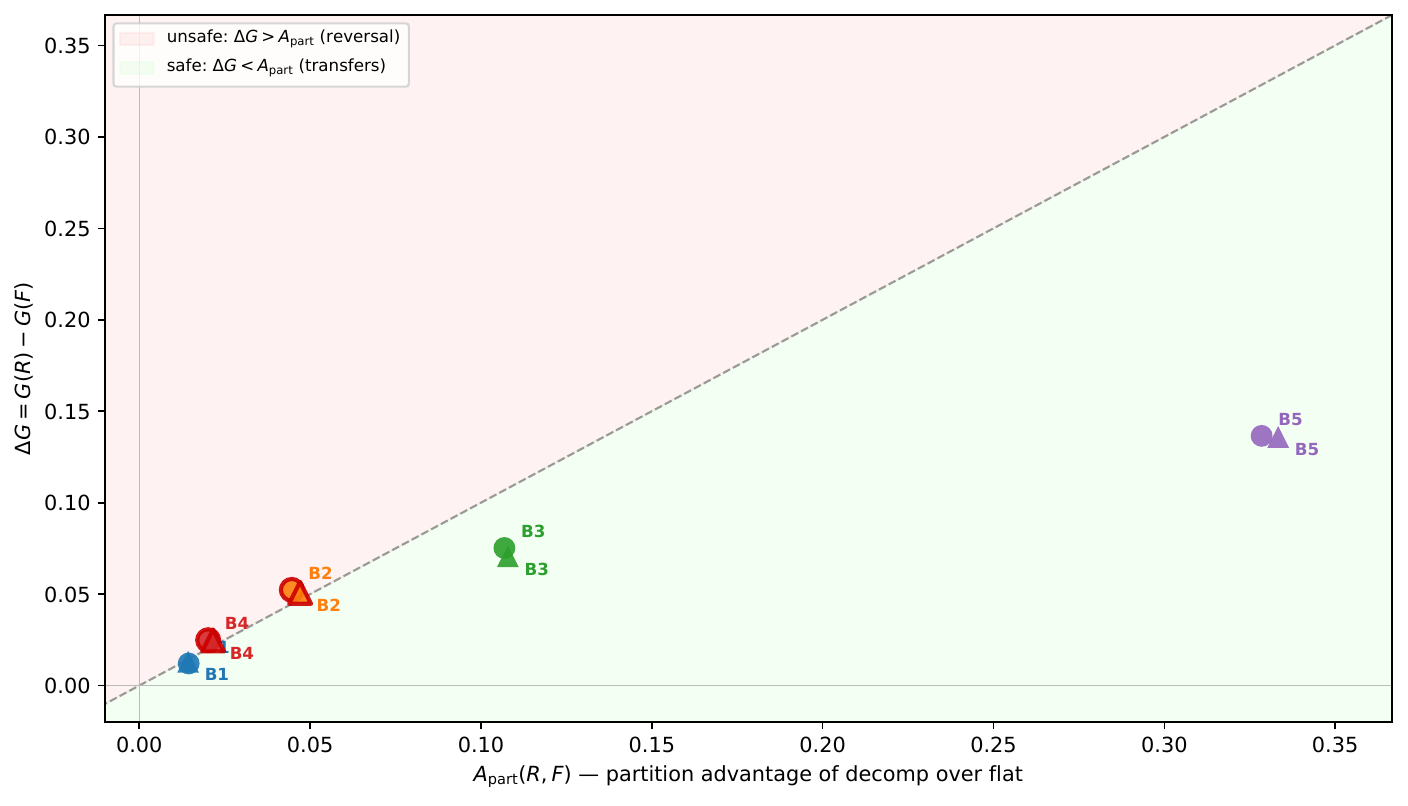}
\caption{Deployment-safety plot for the $10$ decomp-vs-flat
comparisons. By Corollary~\ref{cor:flat-safety} (general form), the
partition-level advantage transfers iff $\Apart > \Delta G$, where
$\Delta G = G(R) - G(F)$, i.e.\ below the dashed $y=x$ line. B1, B2,
B4 occupy distinct absorption regimes: B2 and B4 lie in the unsafe
(red) region and have sign-changing point estimates; B1 remains deployment-safe but is almost
erased. B3 and B5 (PROTEUS-2014) are deployment-safe but partially
absorbed; B5's $33$-point partition advantage shrinks to a $20$-point
end-to-end advantage, $\rho \approx 0.41$. All flat controls have
$G(F)=0$.}
\label{fig:safety}
\end{figure}

\begin{table}[H]
\small
\centering
\caption{Partition-only identification intervals
(Theorem~\ref{thm:identification}) summary. ``Crosses 0'' counts the
fraction of pairs whose interval strictly crosses zero, i.e.\ the fraction of
comparisons on which the partition-level report alone (with
selected-family range information) does not certify a deployable
winner. Median interval width $\bar{W} = (W_{R_1} + W_{R_2})/2$ is the
half-interval that the partition-level report cannot resolve.}
\label{tab:interval-summary}
\begin{tabular}{lrrrr}
\toprule
ID & Pairs & Crosses 0 & Median $\bar{W}$ & Mean $|\Apart|$ \\
\midrule
B1 & 3 & 3 & $0.363$ & $0.010$ \\
B2 & 3 & 3 & $0.102$ & $0.031$ \\
B3 & 3 & 3 & $0.233$ & $0.066$ \\
B4 & 3 & 3 & $0.134$ & $0.014$ \\
B5 & 3 & 3 & $0.378$ & $0.221$ \\
\bottomrule
\end{tabular}
\end{table}

\begin{table}[H]
\small
\centering
\caption{Per-instance margin--regret diagnostic rates for decomposed
pipelines. The reported quantity is the selected-family existential
violation rate from Theorem~\ref{thm:margin-regret}; bootstrap
intervals are omitted here for space and are included in
\texttt{results/exp9\_margin\_regret}.}
\label{tab:margin-regret-rates}
\begin{tabular}{lrrr}
\toprule
ID & $n$ & pairwise & multiclass \\
\midrule
B1 & $133$ & $0.248$ & $0.263$ \\
B2 & $202$ & $0.391$ & $0.371$ \\
B3 & $192/191$ & $0.557$ & $0.545$ \\
B4 & $556$ & $0.228$ & $0.234$ \\
B5 & $3565$ & $0.381$ & $0.387$ \\
\bottomrule
\end{tabular}
\end{table}

\begin{figure}[t]
\centering
\includegraphics[width=\linewidth]{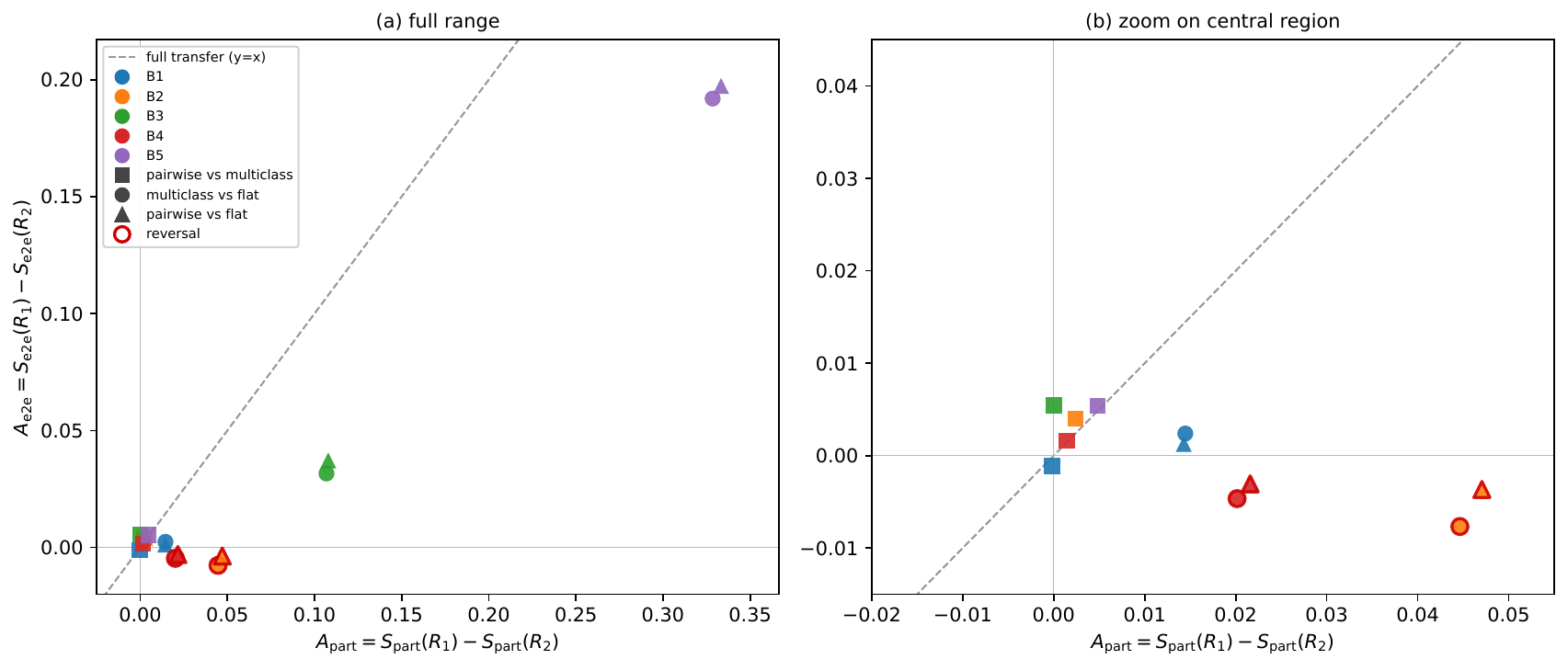}
\caption{Partition-level advantage $\Apart$ against end-to-end
advantage $\Aetoe$ across all $15$ pairwise method comparisons (color
by benchmark, shape by pair type). The dashed diagonal is full transfer
($\Apart = \Aetoe$); points near the horizontal axis indicate
absorption; points in opposite quadrants indicate sign changes. These
cells are outlined in red. \emph{(a)} full range showing B5's outlier
position; \emph{(b)} zoom on the central region containing the four
sign-changing cells (decomp-vs-flat on B2 and B4 --- all in the
lower-right quadrant, where partition prefers decomposition but e2e
prefers flat), the high-absorption B1 cells, and the near-origin
pairwise-vs-multiclass comparisons.}
\label{fig:scatter}
\end{figure}
Pairwise-vs-multiclass entries with
$|\Apart| < 5\times 10^{-3}$ are marked NA for $\rho$ per the reporting threshold
in \S\ref{sec:guidance}.

\begin{table}[H]
\small
\centering
\caption{Full deployment-fidelity-gap table by
(benchmark, decomposition). $G = \Spart - \Setoe$; CI is a $95\%$
percentile-bootstrap interval over per-instance gaps; $p$ is the paired
Wilcoxon $p$-value comparing $u_{\mathrm{part}}$ vs $u_{\mathrm{e2e}}$
on the same instances.}
\label{tab:exp2-full}
\begin{tabular}{llrrrlrr}
\toprule
ID & Decomposition & $\Spart$ & $\Setoe$ & $G$ & $95\%$ CI & $p$ & $n$ \\
\midrule
B1 & pairwise   & $0.9284$ & $0.9154$ & $0.0130$ & $[0.0091,\, 0.0176]$ & $8.15{\times}10^{-17}$ & $133$ \\
B1 & multiclass & $0.9286$ & $0.9165$ & $0.0121$ & $[0.0087,\, 0.0162]$ & $5.57{\times}10^{-17}$ & $133$ \\
B1 & flat       & $0.9141$ & $0.9141$ & $0.0000$ & $[0.0000,\, 0.0000]$ & $1$                     & $133$ \\
B2 & pairwise   & $0.9508$ & $0.9001$ & $0.0507$ & $[0.0401,\, 0.0625]$ & $1.96{\times}10^{-30}$ & $202$ \\
B2 & multiclass & $0.9484$ & $0.8961$ & $0.0523$ & $[0.0414,\, 0.0645]$ & $4.37{\times}10^{-30}$ & $202$ \\
B2 & flat       & $0.9037$ & $0.9037$ & $0.0000$ & $[0.0000,\, 0.0000]$ & $1$                     & $202$ \\
B3 & pairwise   & $0.8361$ & $0.7654$ & $0.0707$ & $[0.0522,\, 0.0905]$ & $1.72{\times}10^{-31}$ & $194$ \\
B3 & multiclass & $0.8357$ & $0.7605$ & $0.0752$ & $[0.0564,\, 0.0958]$ & $3.43{\times}10^{-31}$ & $193$ \\
B3 & flat       & $0.7247$ & $0.7247$ & $0.0000$ & $[0.0000,\, 0.0000]$ & $1$                     & $200$ \\
B4 & pairwise   & $0.9962$ & $0.9716$ & $0.0247$ & $[0.0153,\, 0.0348]$ & $3.90{\times}10^{-91}$ & $556$ \\
B4 & multiclass & $0.9948$ & $0.9700$ & $0.0248$ & $[0.0155,\, 0.0350]$ & $2.68{\times}10^{-91}$ & $556$ \\
B4 & flat       & $0.9746$ & $0.9746$ & $0.0000$ & $[0.0000,\, 0.0000]$ & $1$                     & $556$ \\
B5 & pairwise   & $0.9543$ & $0.8183$ & $0.1360$ & $[0.1257,\, 0.1467]$ & $<10^{-300}$            & $3565$ \\
B5 & multiclass & $0.9495$ & $0.8129$ & $0.1365$ & $[0.1262,\, 0.1472]$ & $<10^{-300}$            & $3565$ \\
B5 & flat       & $0.6210$ & $0.6210$ & $0.0000$ & $[0.0000,\, 0.0000]$ & $1$                     & $3565$ \\
\bottomrule
\end{tabular}
\end{table}

\begin{table}[H]
\small
\centering
\caption{Full advantage-transfer table: all $15$
(benchmark, pair) comparisons. ``pw'' = pairwise OVO, ``mc'' = multiclass
softmax. $\rho = (G(R_1) - G(R_2))/\Apart$ marked NA when
$|\Apart| < 5\times 10^{-3}$. The final column marks sign-changing
point estimates, not a large-margin decision threshold.}
\label{tab:exp3-full}
\begin{tabular}{llrrrrlc}
\toprule
ID & Pair & $\Apart$ & $\Aetoe$ & $G(R_1)$ & $G(R_2)$ & $\rho$ & Sign chg. \\
\midrule
B1 & pw vs mc   & $-0.0002$ & $-0.0011$ & $0.0130$ & $0.0121$ & NA       & no \\
B1 & mc vs flat & $+0.0145$ & $+0.0024$ & $0.0121$ & $0.0000$ & $+0.835$ & no \\
B1 & pw vs flat & $+0.0143$ & $+0.0013$ & $0.0130$ & $0.0000$ & $+0.912$ & no \\
B2 & pw vs mc   & $+0.0024$ & $+0.0040$ & $0.0507$ & $0.0523$ & NA       & no \\
B2 & mc vs flat & $+0.0446$ & $-0.0077$ & $0.0523$ & $0.0000$ & $+1.172$ & \textbf{yes} \\
B2 & pw vs flat & $+0.0471$ & $-0.0037$ & $0.0507$ & $0.0000$ & $+1.078$ & \textbf{yes} \\
B3 & pw vs mc   & $+0.0000$ & $+0.0054$ & $0.0698$ & $0.0752$ & NA       & no \\
B3 & mc vs flat & $+0.1069$ & $+0.0316$ & $0.0752$ & $0.0000$ & $+0.704$ & no \\
B3 & pw vs flat & $+0.1079$ & $+0.0372$ & $0.0707$ & $0.0000$ & $+0.655$ & no \\
B4 & pw vs mc   & $+0.0014$ & $+0.0016$ & $0.0247$ & $0.0248$ & NA       & no \\
B4 & mc vs flat & $+0.0202$ & $-0.0047$ & $0.0248$ & $0.0000$ & $+1.231$ & \textbf{yes} \\
B4 & pw vs flat & $+0.0216$ & $-0.0031$ & $0.0247$ & $0.0000$ & $+1.142$ & \textbf{yes} \\
B5 & pw vs mc   & $+0.0048$ & $+0.0054$ & $0.1360$ & $0.1365$ & NA       & no \\
B5 & mc vs flat & $+0.3285$ & $+0.1919$ & $0.1365$ & $0.0000$ & $+0.416$ & no \\
B5 & pw vs flat & $+0.3333$ & $+0.1973$ & $0.1360$ & $0.0000$ & $+0.408$ & no \\
\bottomrule
\end{tabular}
\end{table}

\section{Statistical methodology}
\label{app:stats}

\paragraph{Pairing unit.} The pairing unit for all paired tests is the
\emph{test instance}, not the (seed, instance) pair: seeds for the same
instance share folds and are not i.i.d. We average per-instance
utilities across seeds first, then perform paired tests on the
per-instance means.

\paragraph{Paired Wilcoxon tests.} For each (benchmark, decomposition)
cell we run a two-sided Wilcoxon signed-rank test on the per-instance
gap vector $u_{\mathrm{part}} - u_{\mathrm{e2e}}$. Reported $p$-values
are exact when ties allow, otherwise asymptotic. Across $15$ cells the
Holm-Bonferroni adjusted $p$-values remain $< 10^{-15}$ for all
decomposed cells.

\paragraph{Bootstrap confidence intervals.} Confidence intervals on
$G(R)$ are computed by sampling $B = 10\,000$ bootstrap replicates of
the per-instance gap vector with replacement and reporting the
$2.5$th and $97.5$th percentiles of the bootstrap mean
distribution. The same convention is applied to $\Apart, \Aetoe$
when reported with intervals.

\paragraph{Effect size.} For pairwise method comparisons we report
$\rho$ as the primary effect-size summary; $\rho$ is dimensionless and
combines the magnitude of $\Apart$ with the signed gap-difference
$G(R_1) - G(R_2)$. We do not report Cliff's $\delta$ in the main paper
because $\rho$ has a direct algebraic interpretation in terms of
absorption.

\section{Compute resources and reproducibility}
\label{app:compute}

\paragraph{Hardware.} All experiments run on a single 8-core CPU
workstation; no GPU is required. Sklearn's \texttt{n\_jobs=-1} is used
for random-forest training. Memory footprint stays below $4$\,GB across
all benchmarks.

\paragraph{Wall-clock time.} A full reproduction of the reported analyses
across all five benchmarks takes approximately $3$--$4$ hours of total
wall-clock time on the reference workstation. The most expensive
configuration is the
pairwise-within-family configuration on TALENT (B3), which trains
$O(|F|^2 \cdot K_{\mathrm{cv}} \cdot \text{seeds})$ binary random
forests per fit; the largest TALENT family contains $25$ algorithms.

\paragraph{Determinism.} Splits and selectors are seeded; rerunning
the scripts with the same seeds reproduces the numerical
tables in this appendix to four decimal places. Bootstrap CIs use
\texttt{numpy.random.default\_rng} with a fixed seed.

\paragraph{Software stack.} Python~3.10, scikit-learn~$\geq$~1.3,
NumPy~$\geq$~1.24, pandas~$\geq$~2.0, SciPy~$\geq$~1.10,
matplotlib~$\geq$~3.7. The companion environment file pins exact
versions used to produce the tables in this paper.

\clearpage
\section{Literature audit of oracle-style reporting}
\label{app:literature-audit}

This appendix substantiates the claim, made in
\S\ref{sec:related} and the abstract, that oracle-style scores --
Virtual Best Solver (VBS), selected-portfolio VBS, virtual-best
encodings, and best-in-family summaries -- are routinely reported in
algorithm selection, solver portfolios, encoding selection, and
tabular benchmarking. Table~\ref{tab:literature-audit} lists, for each
of $12$ representative anchor citations, the oracle-style reporting
object that appears in the work, what that row supports about the
prevalence of such reporting, and an explicit caveat indicating what
the row does \emph{not} establish.

The audit makes a deliberately narrow empirical claim. It does
\emph{not} claim that any cited work conflates a partition-level oracle
score with a deployable system score: each cited paper either reports
the oracle quantity as an explicit upper bound (e.g.\ AutoFolio's
oracle/VBS columns;
Proteus's VB-CSP, VB-SAT, and VB-Encoding rows reported alongside the
deployable Proteus row) or studies the gap between such an oracle and
a learned selector at a single decision level (e.g.\ encoding
selection). The claim the audit \emph{does} support is that an
oracle-style reporting object -- a virtual-best at some level, or a
best-in-family summary -- is the standard primary comparison object
across all four sub-areas. Combined with our empirical results, this
shows that the deployment-fidelity gap defined in this paper is a
gap that the existing reporting practice cannot diagnose: when a
reported number takes the within-family maximum after a family choice
is made, the gap to a deployable end-to-end system that inherits both
choices is not surfaced. \citet{shmuel2025tabular} is the
representative tabular instance for which we have direct supporting
evidence: the published headline counts of which family ``wins'' on
which dataset are derived from the within-family argmax on $10$-fold
cross-validation means, with no within-family deployable selector
specified, defined, or evaluated anywhere in the paper or its
supplement.

\keepXColumns
\footnotesize
\renewcommand{\arraystretch}{1.15}
\begin{tabularx}{\textwidth}{@{}p{2.85cm} X X c@{}}
\caption{Literature audit of oracle-style reporting. ``Strength''
indicates how directly the row supports the audit's narrow claim that
oracle-style quantities (VBS, selected-portfolio VBS, virtual-best
encoding, best-in-family) are pervasive across algorithm selection,
solver portfolios, encoding selection, and tabular benchmarking
(\protect\rule{0.45em}{0.45em}\protect\rule{0.45em}{0.45em}\protect\rule{0.45em}{0.45em}=direct,
\protect\rule{0.45em}{0.45em}\protect\rule{0.45em}{0.45em}=supporting,
\protect\rule{0.45em}{0.45em}=adjacent context). The caveat column
states what each row does \emph{not} establish.}
\label{tab:literature-audit} \\
\toprule
\textbf{Anchor} & \textbf{Oracle-style reporting object} & \textbf{Audit support / caveat} & \textbf{Str.} \\
\midrule
\endfirsthead
\toprule
\textbf{Anchor} & \textbf{Oracle-style reporting object} & \textbf{Audit support / caveat} & \textbf{Str.} \\
\midrule
\endhead
\midrule
\multicolumn{4}{r}{\emph{continued on next page}}\\
\endfoot
\bottomrule
\endlastfoot
\citet{cameron2016vbs} & VBS defined as a hypothetical algorithm picking the best solver per instance from a portfolio; SBS--VBS gap analysed for evaluator bias under randomised solvers. & Standard VBS diagnostic; bias paper for the full-portfolio VBS, not a within-family nested oracle. \emph{Caveat:} not a decomposed AS critique. & \rule{0.45em}{0.45em}\rule{0.45em}{0.45em}\rule{0.45em}{0.45em} \\
\citet{bischl2016aslib} & ASlib repository: AS scenarios benchmark deployable selectors against VBS and SBS as standardised references. & VBS/SBS as the AS-benchmark-ecosystem standard. \emph{Caveat:} ASlib does not require nor prohibit reporting nested oracle quantities; it standardises a single-stage comparison. & \rule{0.45em}{0.45em}\rule{0.45em}{0.45em}\rule{0.45em}{0.45em} \\
\citet{lindauer2015autofolio} & AutoFolio reports oracle/VBS as columns alongside the deployable selector in its result tables. & Joint reporting of deployable selector and oracle bound is the canonical AS format. \emph{Caveat:} AutoFolio explicitly labels oracle as a bound, not a deployable score. & \rule{0.45em}{0.45em}\rule{0.45em}{0.45em}\rule{0.45em}{0.45em} \\
\citet{lindauer2019competitions} & The 2015/2017 AS competitions normalise scores by the SBS--VBS interval. & VBS is a competition-protocol primitive, not just a per-paper choice. \emph{Caveat:} normalisation is single-level, not nested. & \rule{0.45em}{0.45em}\rule{0.45em}{0.45em} \\
\citet{hurley2014proteus} & Hierarchical solver portfolio: CSP-vs-SAT, then SAT-encoding, then SAT solver. Tables report VB-Proteus, VB-CSP, VB-SAT, plus per-encoding VB-DirectOrder, VB-Direct, VB-Support alongside the deployable Proteus selector. & Direct evidence that decomposed solver portfolios report \emph{subportfolio} virtual-best scores. \emph{Caveat:} Proteus does not claim its subportfolio VBs are deployable; it reports the deployable system in the same table. The deployment-fidelity gap is precisely the difference between the rows. & \rule{0.45em}{0.45em}\rule{0.45em}{0.45em}\rule{0.45em}{0.45em} \\
\citet{bach2022kportfolios} & SAT solver $k$-portfolios: VBS of a selected portfolio is reported as a lower bound (i.e.\ best achievable cost) for any model-based selector built on that portfolio. & Selected-portfolio VBS is a normal evaluation object. \emph{Caveat:} this is portfolio analysis, not a decomposed AS pipeline. & \rule{0.45em}{0.45em}\rule{0.45em}{0.45em}\rule{0.45em}{0.45em} \\
\citet{kostovska2023psaas} & Black-box optimisation AAS: AAS performance is reported relative to the virtual best solver \emph{from the selected portfolio}, after a portfolio-selection step. & Recent (2023) anchor: selected-portfolio VBS is still a primary comparator in AAS literature. \emph{Caveat:} portfolio selection is a different stage from the within-family selection studied here. & \rule{0.45em}{0.45em}\rule{0.45em}{0.45em} \\
\citet{stojadinovic2014mesat} & CSP-to-SAT encoding selection: motivates choosing among direct, log, support, order encodings using CSP-side syntactic features because no encoding dominates uniformly. & Encoding-level family selection is a real research problem; the family abstraction in our framework maps directly onto the encoding choice. \emph{Caveat:} meSAT is encoding-selection methodology, not an oracle-vs-deployable critique. & \rule{0.45em}{0.45em} \\
\citet{ulricholtean2022encodings} & Pseudo-Boolean / linear-integer constraint encoding selection: virtual-best encoding reported as the upper bound for a learned encoding selector. & Per-instance virtual-best encoding is the standard upper bound in encoding selection. \emph{Caveat:} a single-level oracle, evaluated correctly as a bound. & \rule{0.45em}{0.45em}\rule{0.45em}{0.45em} \\
\citet{ulricholtean2023learning} & Journal extension: explicitly analyses the single-best vs virtual-best encoding gap and reports how much of that gap a supervised encoding selector closes. & Single-best/virtual-best gap is now a primary analysis object in the encoding-selection literature; close to the spirit of $G(R)$ but applied to a flat encoding choice rather than a decomposed selector. \emph{Caveat:} not a decomposed AS pipeline. & \rule{0.45em}{0.45em}\rule{0.45em}{0.45em}\rule{0.45em}{0.45em} \\
\citet{mcelfresh2023tabzilla} & TabZilla: trains a meta-model to predict whether the best neural network outperforms the best gradient boosting model on a given dataset; analyses choosing the best algorithm \emph{family} versus tuning a single model. & Best-in-family is the primary comparison object in tabular benchmarking. \emph{Caveat:} an analysis paper, not an AS-system claim. & \rule{0.45em}{0.45em} \\
\citet{shmuel2025tabular} & Neurocomputing 2025: $111$ datasets $\times$ $20$ models $\times$ $10$ folds. Headline ranking tables report each model's number of datasets where it is best-in-family on the CV mean. The meta-learner label $\bar{Y}_i = \mathbf{1}[\max_{m\in\mathrm{ML}} \mathrm{score}_i(m) > \max_{m\in\mathrm{DL}} \mathrm{score}_i(m)]$ takes the within-family maximum on the same scores. No within-family deployable selector is defined. & \textbf{Named-instance evidence: a partition-level oracle quantity is the headline conclusion object} in a peer-reviewed tabular benchmark, with no deployable counterpart specified anywhere in paper or supplement. \emph{Caveat:} the paper makes no AS claim and does not advocate deploying these summaries; the omission is the absence of a within-family selector, not the misuse of one. & \rule{0.45em}{0.45em}\rule{0.45em}{0.45em}\rule{0.45em}{0.45em} \\
\end{tabularx}

\paragraph{What the audit does and does not buy.} The audit gives the
abstract's strong claim a defensible empirical grounding: oracle-style
reporting objects exist at every level of the AS / portfolio / encoding /
tabular pipeline, and they are the primary comparison objects in those
sub-areas. What it does \emph{not} do is allege misuse by any cited
work. The contribution of the present paper is therefore not a critique
of prior reporting but the construction of the missing object -- the
deployable system that inherits the partition's family choice -- and
the measurement of the gap $G(R)$ between the partition-level oracle
score and that system's score. The five most directly load-bearing
anchors for our framing are \citet{cameron2016vbs} (VBS as the canonical
AS oracle diagnostic), \citet{bischl2016aslib} (VBS in the
benchmark-ecosystem standard), \citet{lindauer2015autofolio}
(deployable-selector vs.\ oracle joint reporting),
\citet{hurley2014proteus} (hierarchical decomposed portfolio with
subportfolio virtual-best scores), and either
\citet{ulricholtean2023learning} or \citet{shmuel2025tabular} for a
recent, named instance of single-best vs.\ virtual-best comparison
respectively in encoding selection and tabular benchmarking.

\clearpage
\section{Inner-selector stress-test full results}
\label{app:spectrum}

This appendix provides the per-cell numbers backing the
\emph{Inner-selector stress test} paragraph in \S\ref{sec:robustness}.
We sweep seven within-family selector classes spanning
single-best-on-train, the paper-default per-algorithm GBDT, a tuned
sklearn GBDT, sklearn RandomForest, sklearn MLP, LightGBM, and
XGBoost, all under the multiclass family selector and the same
$5$-fold split protocol as the headline runs (here with $2$ seeds;
the headline runs use $5$). Hyperparameters: tuned GBDT $n=300, d=5$;
RF $n=200, d=10$; MLP hidden $(64,32)$, max iter $200$; LightGBM
$n=200, \text{leaves}=31, \mathrm{lr}=0.05$; XGBoost $n=200, d=6,
\mathrm{lr}=0.05$. The package's existing \texttt{WithinFamilySelector}
contract is reused; the new selector classes are declared inline in
the experiment script and do not modify any package code.

The summary below reports $(\Spart, \Setoe, G)$ per cell. The
within-family oracle ($\underline{G} = 0$) is omitted as the trivial
upper bound on each row; $\Spart$ is invariant across within-classes
within a benchmark by construction (Lemma~\ref{lem:gap-identity}) and
serves as a sanity check.
We use this appendix only as a $G$-robustness summary: decision-level
decomp-vs-flat verdicts are reported in Table~\ref{tab:transfer} and
Table~\ref{tab:exp3-full} under the common-instance, singleton-flat
protocol. The six learned within-family regressors are separated from
the non-learned SingleBest reference because SingleBest is a useful
stress case but not the learned selector class used in the main runs.

\begin{center}
\small
\textbf{Inner-selector spectrum, decomposed pipeline.} Per-cell
$\Spart, \Setoe, G$ under the paper's primary multiclass family
selector. For each benchmark the six learned within-family regressors
are listed first, followed by the non-learned SingleBest reference.
$\Spart$ is constant within a benchmark by
Lemma~\ref{lem:gap-identity}.
\par\vspace{0.4em}
\begin{tabular}{llrrr}
\toprule
ID & Within-class & $\Spart$ & $\Setoe$ & $G$ \\
\midrule
B1 & PerAlgRF             & $0.9287$ & $0.9205$ & $0.0082$ \\
B1 & PerAlgGBDT\_default  & $0.9287$ & $0.9178$ & $0.0109$ \\
B1 & PerAlgXGBoost        & $0.9287$ & $0.9143$ & $0.0143$ \\
B1 & PerAlgLightGBM       & $0.9287$ & $0.9137$ & $0.0150$ \\
B1 & PerAlgGBDT\_tuned    & $0.9287$ & $0.9090$ & $0.0197$ \\
B1 & PerAlgMLP            & $0.9287$ & $0.9052$ & $0.0235$ \\
B1 & SingleBest           & $0.9287$ & $0.9185$ & $0.0102$ \\
\midrule
B2 & PerAlgRF             & $0.9464$ & $0.9011$ & $0.0453$ \\
B2 & PerAlgXGBoost        & $0.9464$ & $0.8966$ & $0.0498$ \\
B2 & PerAlgLightGBM       & $0.9464$ & $0.8964$ & $0.0500$ \\
B2 & PerAlgGBDT\_tuned    & $0.9464$ & $0.8963$ & $0.0501$ \\
B2 & PerAlgGBDT\_default  & $0.9464$ & $0.8941$ & $0.0523$ \\
B2 & PerAlgMLP            & $0.9464$ & $0.8922$ & $0.0542$ \\
B2 & SingleBest           & $0.9464$ & $0.9099$ & $0.0365$ \\
\midrule
B3 & PerAlgLightGBM       & $0.8362$ & $0.7681$ & $0.0681$ \\
B3 & PerAlgRF             & $0.8362$ & $0.7653$ & $0.0709$ \\
B3 & PerAlgGBDT\_default  & $0.8362$ & $0.7582$ & $0.0781$ \\
B3 & PerAlgXGBoost        & $0.8362$ & $0.7575$ & $0.0787$ \\
B3 & PerAlgGBDT\_tuned    & $0.8362$ & $0.7563$ & $0.0799$ \\
B3 & PerAlgMLP            & $0.8362$ & $0.7554$ & $0.0808$ \\
B3 & SingleBest           & $0.8362$ & $0.7433$ & $0.0929$ \\
\midrule
B4 & PerAlgGBDT\_default  & $0.9962$ & $0.9740$ & $0.0222$ \\
B4 & PerAlgLightGBM       & $0.9962$ & $0.9725$ & $0.0237$ \\
B4 & PerAlgMLP            & $0.9962$ & $0.9722$ & $0.0240$ \\
B4 & PerAlgRF             & $0.9962$ & $0.9715$ & $0.0247$ \\
B4 & PerAlgXGBoost        & $0.9962$ & $0.9697$ & $0.0265$ \\
B4 & PerAlgGBDT\_tuned    & $0.9962$ & $0.9695$ & $0.0267$ \\
B4 & SingleBest           & $0.9962$ & $0.9662$ & $0.0300$ \\
\midrule
B5 & PerAlgMLP            & $0.9491$ & $0.8162$ & $0.1329$ \\
B5 & PerAlgGBDT\_default  & $0.9491$ & $0.8134$ & $0.1357$ \\
B5 & PerAlgRF             & $0.9491$ & $0.8130$ & $0.1361$ \\
B5 & PerAlgLightGBM       & $0.9491$ & $0.8078$ & $0.1413$ \\
B5 & PerAlgXGBoost        & $0.9491$ & $0.8076$ & $0.1414$ \\
B5 & PerAlgGBDT\_tuned    & $0.9491$ & $0.8059$ & $0.1431$ \\
B5 & SingleBest           & $0.9491$ & $0.6883$ & $0.2608$ \\
\bottomrule
\end{tabular}
\end{center}

After this full spectrum, we also reran targeted advantage-transfer
checks aligned with the headline decomp-vs-flat comparisons. On B2/B4,
a $5$-seed sweep with RF and LightGBM under both multiclass and
pairwise family selectors keeps $G(R)>0$ in every cell (B2:
$0.046$--$0.049$; B4: $0.024$--$0.027$). The corresponding
$\Aetoe$ values remain near zero (B2: $-0.004$ to $+0.001$; B4:
$-0.006$ to $-0.003$), so the strict sign of these small-margin cells
is selector-sensitive, but no stronger selector yields a reliable
positive decomp-vs-flat advantage. A lighter $3$-seed B5 default-vs-RF
rerun keeps $\rho$ at $0.407$--$0.414$. The outputs are in
\texttt{results/exp\_inner\_selector\_spectrum\_stronger\_2026\_05\_05\_b2b4}
and
\texttt{results/exp\_inner\_selector\_spectrum\_stronger\_2026\_05\_05\_b5\_light}.

\paragraph{Reading the table.} Across the six learned within-family
regressors, the spread of $G$ is modest within each benchmark: B1
$0.0082$--$0.0235$, B2 $0.0453$--$0.0542$, B3 $0.0681$--$0.0808$, B4
$0.0222$--$0.0267$, and B5 $0.1329$--$0.1431$. The non-learned
SingleBest reference should be read separately: it is competitive on
B1/B2, worse on B3/B4, and much worse on B5, where it yields
$G=0.2608$ (26.1 pp). This 26 pp SingleBest stress case is distinct
from the 21 pp maximum in Figure~\ref{fig:robust}, which comes from
the main robustness sweep over the paper's within-family selector
classes. The learned-selector spectrum therefore supports the
robustness claim without relying on the stronger SingleBest stress
case.

\end{document}